\documentclass[]{elsarticle}

\usepackage{lineno,hyperref}
\modulolinenumbers[5]

\makeatletter
\def\ps@pprintTitle{%
    \let\@oddhead\@empty
    \let\@evenhead\@empty
    \let\@oddfoot\@empty
    \let\@evenfoot\@empty
    }
\makeatother
\biboptions{authoryear}

\usepackage{algorithm}
\usepackage{algorithmic}
\usepackage{newfloat}
\usepackage{listings}
\floatstyle{ruled}
\newfloat{listing}{tb}{lst}{}
\floatname{listing}{Listing}
\usepackage{latexsym,amssymb,amsmath,amsthm}
\usepackage{xcolor}
\usepackage{xspace}
\usepackage{url}
\usepackage{enumerate}
\usepackage{comment}
\usepackage{todonotes}
\usepackage{stackengine}
\usepackage{subcaption}
\usepackage{booktabs}

\usepackage[inline]{enumitem}
\newlist{romanenumerate*}{enumerate*}{1}
\setlist[romanenumerate*]{label=(\textit{\roman*})}
\newlist{romanenumerate}{enumerate}{1}
\setlist[romanenumerate]{label=(\textit{\roman*})}

\newtheorem{theorem}{Theorem}
\newtheorem{lemma}{Lemma}

\newtheorem{corollary}{Corollary}
\newtheorem{definition}{Definition}
\newtheorem{example}{Example}

\newcommand{\astfootnote}[1]{%
\let\oldthefootnote=\thefootnote%
\setcounter{footnote}{0}%
\renewcommand{\thefootnote}{\fnsymbol{footnote}}%
\footnote{#1}%
\let\thefootnote=\oldthefootnote%
}

\usepackage{tikz}
\usetikzlibrary{automata, positioning, arrows, fit}
\tikzset{->,    
    >=stealth,
    node distance=3cm,
    every state/.style={thick, fill=gray!10},
    align=center
}

\makeatletter 
\newcommand{\mylabel}[2]{#2\def\@currentlabel{#2}\label{#1}}
\makeatother

\newcommand{\blocks}{{\textsc{BlocksWorld}}\xspace}
\newcommand{\blocksnd}{{\textsc{BlocksWorldND}}\xspace}
\newcommand{\tireworld}{{\textsc{TriangleTireworld}}\xspace}

\newcommand{\elevator}{{\textsc{Elevator}}\xspace}

\newcommand{\planforpast}{{\ensuremath{\mathsf{Plan4Past}}}\xspace}
\newcommand{\planforpastshort}{{\ensuremath{\mathsf{P4P}}}\xspace}
\newcommand{\fondforlp}{{\ensuremath{\mathsf{FOND4LTLf}}}\xspace}
\newcommand{\fondforlpshort}{{\ensuremath{\mathsf{F4LP}}}\xspace}
\newcommand{\ltlfondtofond}{{\ensuremath{\mathsf{LTLFOND2FOND}}}\xspace}
\newcommand{\ltlfondtofondshort}{{\ensuremath{\mathsf{LF2F}}}\xspace}

\newcommand{\fastdownward}{{\ensuremath{\mathsf{FastDownward}}}\xspace}
\newcommand{\fastdownwardshort}{{\ensuremath{\mathsf{FD}}}\xspace}
\newcommand{\mynd}{{\ensuremath{\mathsf{MyND}}}\xspace}
\newcommand{\myndshort}{{\ensuremath{\mathsf{MyND}}}\xspace}
\newcommand{\prp}{{\ensuremath{\mathsf{PRP}}}\xspace}

\newcommand{\planforpastfd}{{\ensuremath{\mathsf{\planforpast+\fastdownward}}}\xspace}
\newcommand{\planforpastfdshort}{{\ensuremath{\mathsf{\planforpastshort+\fastdownwardshort}}}\xspace}
\newcommand{\planforpastmynd}{{\ensuremath{\mathsf{\planforpast+\mynd}}}\xspace}
\newcommand{\planforpastmyndshort}{{\ensuremath{\mathsf{\planforpastshort+\myndshort}}}\xspace}
\newcommand{\fondforlpfdshort}{{\ensuremath{\mathsf{\fondforlpshort+\fastdownwardshort}}}\xspace}

\newcommand{\fondforlpmyndshort}{\ensuremath{\mathsf{{\fondforlpshort+\myndshort}}}\xspace}
\newcommand{\ltlfondtofondfdshort}{\ensuremath{\mathsf{{\ltlfondtofondshort+\fastdownwardshort}}}\xspace}

\newcommand{\ltlfondtofondmyndshort}{{\ensuremath{\mathsf{\ltlfondtofondshort+\myndshort}}}\xspace}

\newcommand{\start}{\mathsf{start}\xspace}

\newcommand{\policy}{\pi}

\newcommand{\val}{{\mathsf{val}}\xspace}
\newcommand{\pl}{{\mathsf{pl}}\xspace}
\newcommand{\prog}{{\mathsf{prog}}\xspace}
\newcommand{\ev}{{\mathsf{eval}}\xspace}

\newcommand{\quoted}[1]{``{#1}"}
\newcommand{\last}{\mathsf{last}\xspace}

\newcommand{\length}{\mathsf{length}\xspace}
\newcommand{\plus}[1]{\mathsf{{#1}^+}}
\newcommand{\prop}{{\mathsf{prop}}\xspace}
\newcommand{\sub}{{\mathsf{sub}}\xspace}
\newcommand{\Sphi}{{\mathsf{\Sigma_\varphi}}\xspace}
\newcommand{\Sp}{{\mathsf{\Sigma_\P}}\xspace}
\newcommand{\sigmaPhi}[1]{{\sigma_{#1}^{\varphi}}\xspace}
\newcommand{\sigmaY}[1]{{\sigma_{#1}}\xspace}
\newcommand{\s}{{s}\xspace}
\newcommand{\sigmaPhiprime}{\sigma^{\varphi\prime}\xspace}
\newcommand{\pred}[1]{\texttt{#1}\xspace}

\newcommand{\trace}{{\tau}\xspace}
\newcommand{\Plus}{{[\varphi]}\xspace}
\newcommand{\tracePlus}{{\tau^\Plus}\xspace}
\newcommand{\traceInd}{{\tau_{n{-}1}{\cdot}\s_n}\xspace}

\newcommand{\Pre}{{\mathit{Pre}}\xspace}
\newcommand{\Eff}{{\mathit{Eff}}\xspace}
\newcommand{\Val}{{\mathsf{Val}_\varphi}\xspace}

\newcommand{\PA}{{\mathsf{pa}}\xspace}

\newcommand{\pnf}{{\mathsf{pnf}}\xspace}
\newcommand{\PNF}{{\sc pnf}\xspace}

\newcommand{\myi}{(\emph{i})\xspace}
\newcommand{\myii}{(\emph{ii})\xspace}
\newcommand{\myiii}{(\emph{iii})\xspace}

\newcommand{\A}{\mathcal{A}} 
 \newcommand{\D}{\mathcal{D}}
 \newcommand{\F}{\mathcal{F}}

 \renewcommand{\P}{\mathcal{P}}
 \newcommand{\R}{\mathcal{R}}

\newcommand{\Rp}{{\R}_{\textrm{past}}}

\newcommand{\defeq}{=}

\newcommand{\Yesterday}{{\mathsf{Y}}\xspace}
\newcommand{\Wyesterday}{{\mathsf{WY}}\xspace}
\newcommand{\Historically}{{\mathsf{H}}\xspace}
\newcommand{\Once}{{\mathsf{O}}\xspace}
\newcommand{\Since}{\,{\mathsf{S}}\,\xspace}

\newcommand{\Always}{{\mathsf{G}}\xspace}
\newcommand{\Eventually}{{\mathsf{F}}\xspace}
\newcommand{\Next}{{\mathsf{X}}\xspace}
\newcommand{\Wnext}{{\mathsf{WX}}\xspace}
\newcommand{\lUntil}{\,{\mathsf{U}}\,\xspace}
\newcommand{\Release}{\,{\mathsf{R}}\,\xspace}

\makeatletter
\newsavebox{\@brx}
\newcommand{\llangle}[1][]{\savebox{\@brx}{\(\m@th{#1\langle}\)}%
  \mathopen{\copy\@brx\kern-0.5\wd\@brx\usebox{\@brx}}}
\newcommand{\rrangle}[1][]{\savebox{\@brx}{\(\m@th{#1\rangle}\)}%
  \mathclose{\copy\@brx\kern-0.5\wd\@brx\usebox{\@brx}}}
\makeatother
\makeatletter
\newsavebox{\@sqbrx}
\newcommand{\llbracket}[1][]{\savebox{\@sqbrx}{\(\m@th{#1\lbrack}\)}%
  \mathopen{\copy\@sqbrx\kern-0.5\wd\@sqbrx\usebox{\@sqbrx}}}
\newcommand{\rrbracket}[1][]{\savebox{\@sqbrx}{\(\m@th{#1\rbrack}\)}%
  \mathclose{\copy\@sqbrx\kern-0.5\wd\@sqbrx\usebox{\@sqbrx}}}
\makeatother

\newcommand{\true}{\mathit{true}}

\newcommand{\false}{\mathit{false}}
\newcommand{\ttrue}{\textsf{true}}
\newcommand{\ffalse}{\textsf{false}}
\newcommand{\Last}{\textsf{end}}
\newcommand{\Started}{\mathit{start}}
\newcommand{\Ended}{\mathit{end}}

\newcommand{\DIAM}[1]{\langle #1 \rangle}
\newcommand{\BOXP}[1]{ \llbracket #1\rrbracket}
\newcommand{\DIAMP}[1]{\llangle #1 \rrangle}

\newcommand{\LTL}{{\sc ltl}\xspace}
\newcommand{\LTLf}{{\sc ltl}$_f$\xspace}

\newcommand{\LDLf}{{\sc ldl}$_f$\xspace}

\newcommand{\PLTLf}{{\sc ppltl}\xspace}
\newcommand{\PLTLfold}{{\sc pltl}$_f$\xspace}
\newcommand{\PLDLf}{{\sc ppldl}\xspace}

\newcommand{\FOL}{{\sc fol}\xspace}
\newcommand{\MSO}{{\sc mso}\xspace}

\newcommand{\AFA}{{\sc afa}\xspace}
\newcommand{\NFA}{{\sc nfa}\xspace}
\newcommand{\NFAs}{{\sc nfa}s\xspace}

\newcommand{\DFA}{{\sc dfa}\xspace}
\newcommand{\DFAs}{{\sc dfa}s\xspace}

\newcommand{\declare}{{\sc declare}\xspace}

\newcommand{\PDDL}{\textsf{PDDL}\xspace}
\newcommand{\FOND}{FOND\xspace}
\newcommand{\PLTLftitle}{\textbf{\textsc{PPLTL}}\xspace}

\newcommand{\Nat}{{\rm I\kern-.23em N}}
\newcommand{\Prop}{\P}

\newcommand{\tup}[1]{\langle #1 \rangle}

\newcommand{\tiff}{\; \text{ iff }\;}

\begin{document}

\begin{frontmatter}

% \title{Elsevier \LaTeX\ template\tnoteref{mytitlenote}}
% \tnotetext[mytitlenote]{Fully documented templates are available in the elsarticle package on \href{http://www.ctan.org/tex-archive/macros/latex/contrib/elsarticle}{CTAN}.}
\title{Planning for Temporally Extended Goals in Pure-Past Linear Temporal Logic: A Polynomial Reduction to Standard Planning}

% \address{Radarweg 29, Amsterdam}
% \fntext[myfootnote]{Since 1880.}

\author[mymainaddress]{Giuseppe De Giacomo}
\cortext[mycorrespondingauthor]{Corresponding author}
\ead{degiacomo@diag.uniroma1.it}

\author[mymainaddress,mysecondaryaddress]{Marco Favorito\fnref{myfootnote}}
\ead{favorito@diag.uniroma1.it}
\fntext[myfootnote]{Opinions expressed are solely the author's own and do not express the views or opinions of his employer.}

\author[mymainaddress,mythirdaddress]{Francesco Fuggitti\corref{mycorrespondingauthor}}
\ead{fuggitti@diag.uniroma1.it}

\address[mymainaddress]{Sapienza Univesity, Rome, Italy}
\address[mysecondaryaddress]{Bank of Italy}
\address[mythirdaddress]{York University, Toronto, ON, Canada}

\begin{abstract}
We study temporally extended goals expressed in Pure-Past \LTL (\PLTLf). \PLTLf is particularly interesting for expressing goals since it allows to express sophisticated tasks as in the Formal Methods literature, while the worst-case computational complexity of Planning in both deterministic and nondeterministic domains (\FOND) remains the same as for classical reachability goals.
%(Note that this is not the case for planning in nondeterministic domains for LTLf goals.)  
However, while the theory of planning for \PLTLf goals is well understood, practical tools have not been specifically investigated.
In this paper, we make a significant leap forward in the construction of actual tools to handle \PLTLf goals. We devise a technique to polynomially translate planning for \PLTLf goals into standard planning. We show the formal correctness of the translation, its complexity, and its practical effectiveness through some comparative experiments. As a result, our translation enables state-of-the-art tools, such as \fastdownwardshort or \myndshort, to handle \PLTLf goals seamlessly, maintaining the impressive performances they have for classical reachability goals.
\end{abstract}

% \begin{keyword}
% \texttt{elsarticle.cls}\sep \LaTeX\sep Elsevier \sep template
% \MSC[2010] 00-01\sep  99-00
% \end{keyword}

\end{frontmatter}

% \linenumbers

\section{Introduction}

Planning for temporally extended goals has a long tradition in AI Planning, since the pioneering work in the late 90's \citep{Bacchus96,bacchus1997structured,BacchusK98,BacchusK00}, the work on planning via Model Checking \citep{CimattiGGT97,CimattiRT98,DeGiacomoV99,GiunchigliaT99,PistoreT01}, and the work on declarative and procedural constraints \citep{baier2006planning,BaierM06,BaierFBM08} just to mention a few research directions. 
% Moreover, the importance of temporally extended goals is witnessed by the inclusion of trajectory constraints in \textsf{PDDL3} \citep{GereviniHLSD09}.
Moreover, the inclusion of trajectory constraints in \textsf{PDDL3} \citep{GereviniHLSD09} witnesses the importance of temporally extended goals. 

In fact, it is quite compelling to specify agent's tasks (goals) by means of formalisms as Linear-time Temporal Logic (\LTL) that have been advocated as excellent tools to express property of processes by the Formal Methods community \citep{BaKG08}.
On the other hand, in AI Planning tasks need to terminate, thus a finite trace variant of \LTL, namely 
\LTLf, is often more appropriate to specify agent's tasks~\citep{baier2006planning,DegVa13,DegVa15}.
Planning for \LTLf goals has been studied in, e.g.,~\citep{BaierM06,DegVa13,torres2015polynomial} for deterministic domains and in~\citep{DegVa15,camacho2017nondeterministic,degiacomo2018automata} for nondeterministic domains. By now, we have a clear picture. Planning for \LTLf goals is PSPACE-complete for deterministic domain, just like for classical reachability goals \citep{Bylander94}. Whereas, if we turn to nondeterministic domains (\FOND), it is EXPTIME-complete in the domain as for classical reachability goals~\citep{RintanenICAPS04} and 2EXPTIME-complete in the goal.

In deterministic domains, the added expressiveness of \LTLf goals is payed in terms of algorithmic sophistication, but not in worst-case complexity \citep{torres2015polynomial}. While, in nondeterministic domains, the worst-case goal complexity also increases from poly to 2EXPTIME~\citep{degiacomo2018automata}.
% in the case of deterministic domains , and in worst-case goal complexity for nondeterministic domains.
%%
Additional difficulties come from the fact that \LTLf goals can express properties that are non-Markovian \citep{Gabaldon11} and they require to translate the \LTLf formulas into an exponential Nondeterministic Finite-state Automaton (\NFA) and a double-exponential Deterministic Finite-state Automaton (\DFA), respectively in the case of deterministic and nondeterministic domains.

Interestingly, an alternative to \LTLf is the Pure-Past Linear Temporal Logic, or \PLTLf \citep{LichtensteinPZ85,degiacomo2020pure}. This logic looks at the trace backward and expresses non-Markovian properties on traces using past operators. \PLTLf has the same expressive power of \LTLf (although translating \LTLf into \PLTLf and viceversa is in general unfeasible, since the best algorithms known are 3EXPTIME) \citep{degiacomo2020pure}. However, because of a property of reverse languages \citep{CKS81}, the \DFA corresponding to a \PLTLf formula is single exponential, and can be computed directly from the formula, or better, its corresponding Alternating Finite-state Automaton (\AFA) \citep{degiacomo2020pure}. 

In Planning, temporally extended goals expressed in \PLTLf require to  reach a state of affairs,  where a desired \PLTLf formula $\varphi$ holds, i.e. the trace produced to reach such a state of affairs is such that it satisfies the \PLTLf formula $\varphi$.

Our aim is to develop an approach to solve both classical and \FOND planning for \PLTLf goals that sidesteps altogether the construction of \DFA for \PLTLf formula as done, e.g., in \cite{degiacomo2018automata} for \LTLf.
Instead, exploits the key intuitive difference between \LTLf and \PLTLf that given the prefix of the trace computed so far, the \LTLf formula has to consider all possible extensions, while a \PLTLf can simply be evaluated on the history (the prefix of the trace) produced so far. This intuition is at the base of the results in  this paper.
We propose a technique that during the computation of the plan, for each node, the planner also keeps track of the satisfaction of some key subformulas of the goal. In particular, we get inspiration from the classical temporal logic formula progression techniques proposed in \cite{bacchus1997structured,BacchusK98,BacchusK00}, however this time looking at the trace backward in order to evaluate the goal formula over just the current search node, instead of the entire history produced. A similar approach was followed to tackle solving MDP with non Markovian rewards expressed in \PLTLf \citep{Bacchus96,bacchus1997structured}. The resulting technique we get is impressive. We can polynomially translate planning in deterministic domain into classical planning with only a minimal overhead, thus enabling the use of state-of-the art planning, such as \fastdownward~\citep{helmertFD} used in our experiments.  Moreover, exactly the same translation technique allows for solving \FOND for \PLTLf goals by polynomially reducing it to \FOND for classical reachability goals, again with minimal overhead, thus enabling the seamlessly use of state-of-the-art \FOND planning tools, such as \mynd~\citep{mattmuller2010pattern} used in our experiments.

The rest of the paper is organized as follows. In Section~\ref{sec:pltlf}, we give some preliminary notions of \PLTLf. In Section~\ref{sec:planning-for-pltlf-goals}, we introduce the framework of interest in this paper: classical and \FOND planning for temporally extended goals expressed in \PLTLf. In Section~\ref{sec:progression}, we give the key mathematical construction for planning for \PLTLf goals. In Section~\ref{sec:pddl}, we present our reduction technique and we  show how to implement it in \PDDL. In Section~\ref{sec:experiments}, we compare handling \PLTLf goals and classical reachability goals, as well as some comparison with handling \LTLf goals. In Section~\ref{sec:related-work}, we discuss previous related work highlighting similarities and differences. Finally, in Section~\ref{sec:conclusions}, we conclude the paper.

\section{Pure-Past Linear Temporal Logic}
\label{sec:pltlf}
In this section, we introduce the Pure-Past Linear Temporal Logic (\PLTLf). We refer to the survey \cite{degiacomo2020pure} for a more detailed presentation.\footnote{In \cite{degiacomo2020pure} 
\PLTLf is denoted as \PLTLfold.}

Given a set $\Prop$ of propositions, \PLTLf is defined by:
\[\begin{array}{rcl}
\varphi &::=& p \mid \lnot \varphi_1 \mid \varphi_1 \land \varphi_2  \mid \Yesterday\varphi_1 \mid \varphi_1 \Since \varphi_2 
\end{array}
\]
% where $\ttrue$ is the logical true, 
where $p\in \Prop$, 
$\Yesterday$ is the \emph{yesterday} operator and
$\Since$ is the \emph{since} operator.
We define the following common abbreviations:
% the logical false $\ffalse\equiv\lnot\ttrue$,
$\varphi_1\lor\varphi_2 \equiv\lnot(\lnot\varphi_1\land\lnot\varphi_2)$,
the \emph{once} operator
$ \Once\varphi \equiv \true\Since\varphi$, the \emph{historically} operator
$\Historically\varphi \equiv \lnot\Once\lnot\varphi$,
% $start\equiv\Historically\ffalse$ denotes the start of the trace,
% $\propnot{a}\equiv a\wedge \lnot\start$ says that at the current step, $a$ does not hold, 
and the propositional boolean constants $\true\equiv p \lor \lnot p$,
$\false\equiv \lnot true$. Also, $\start \defeq \lnot \Yesterday (\true)$ expresses that the trace has started.
% Note that $\ttrue \neq true$, as the latter does not accept the empty trace. 
% Without loss of generality, we assume that the negation cannot be in front of another negation, as the formula can be rewritten using the equivalence $\lnot\lnot\varphi \equiv \varphi$.

% Here, we define a variant that works also for empty traces.
\PLTLf formulas are interpreted on \emph{finite non-empty} traces, also called \emph{histories},
$\trace = \s_0\cdots \s_{n}$
where $\s_i$ at instant $i$ is a propositional interpretation 
over the alphabet $2^\P$. We denote by $\length(\trace)$ the length
$n+1$ of $\trace$, and 
% by $\slast(\trace)$ and 
by $\last(\trace)$ 
% the second last and 
the last element of the trace.
% , respectively.
Given a trace $\trace = \s_0\cdots \s_{n}$,
we denote by $\trace_{i,j}$, with $0\leq i \leq j \leq n$, the sub-trace $\s_i\dots \s_j$ obtained from $\trace$ starting from position $i$ and ending in position $j$.
% if $j < \length(\trace)$, $\trace_i\dots\trace_{n}$ if $j \ge \length(\trace)$. 
% Note that if $i \ge \length(\trace)$ 
% then $\trace_{i,j}$ denotes the empty trace.
%
We define the satisfaction relation $\trace,i \models \varphi$, stating that $\varphi$ holds at instant $i$, as follows:
	
\begin{itemize}
    % \item $\trace,i \models \ttrue$;
    \item $\trace,i \models p$ iff $\length(\trace)\ge 1$ and $p\in \s_i$ (for $p\in\P)$;
    \item $\trace,i \models \lnot\varphi$ iff $\trace,i\not\models\varphi$;
    \item $\trace,i \models \varphi_1\wedge\varphi_2$ iff $\trace,i\models\varphi_1$ and $\trace,i\models\varphi_2$
    \item $\trace,i \models  \Yesterday \varphi$ \tiff $i\geq1$ and $\trace,{i-1} \models \varphi$;
    \item   $\trace,i \models  \varphi_1 \Since \varphi_2$ \tiff there exists $k$, with $0 \leq k \leq i < \length(\trace)$ such that $\trace,k \models \varphi_2$ and for all $j$, with $ k<j\leq i$, we have that $\trace,j \models \varphi_1$.
\end{itemize}
A \PLTLf formula $\varphi$ is \emph{true} in $\trace$, denoted $\trace \models \varphi$, if  $\trace,  \length(\trace){-}1 \models \varphi$.

\subsection{Examples of \PLTLf Formulas}

We now give several examples of \PLTLf formulas, including the two kinds of formulas that we will later use in our experiments, to demonstrate that \PLTLf is an appropriate and helpful formalism to specify goals in planning.
% \todo[inline]{Give several examples of pure-past temporal formulas on finite traces, including the two forms used in the experiments.}

In many cases, we want the agent to achieve a goal \emph{g} after some condition \emph{c} has been met. In this setting, we identify the \emph{Immendiate-Response} pattern as $g \land \Yesterday(c)$ and the \emph{Bounded-Response} pattern $g \land \Yesterday^i(c)$ for $1\leq i \leq n$, where $n$ is the time bound within which the agent achieves the goal $g$. These and other patterns have been employed in the context of MDP rewards in \cite{Bacchus96}.
Other interesting \PLTLf formulas are $t \land (\lnot a \Since s)$ and $\Historically(b \rightarrow \Yesterday(\lnot b \Since f)$ \citep{degiacomo2020pure}. For instance, the former may state that before achieving task $t$, the agent was not in area $a$ anymore since the area was sanitized (\emph{s}). Whereas, the latter may enforce scenarios such as the one in which the agent has always paid the ticket fee \emph{f} before getting the bus \emph{b}.

Then, among common formulas, we also find the \emph{Strict-Sequence} pattern as $\Once(a \land \Yesterday(\Once(b \land \Yesterday(\Once(\dots)))))$ forcing the agent to achieve tasks $a, b, \dots$ sequencially, and the \emph{Eventually-All} pattern as $\bigwedge_{i=1}^n \Once(a_i)$ requiring to eventually achieve all tasks $a_i$. We use the \emph{Strict-Sequence} and the \emph{Eventually-All} patterns in our experiments later as they easily translate into their corresponding pure-future formulas.

Furthermore, widely used formula patterns can be found in \textsf{PDDL3} (Table~\ref{table:pddl3}) that standardized certain modal operators \citep{GereviniHLSD09} and in \declare (Table~\ref{table:declare}) that is the \emph{de-facto} stardard encoding language for Business Processes behaviors \citep{declareCSRD09}. Table~\ref{table:pddl3} and Table~\ref{table:declare} are both a non-exhaustive list of such common patterns including their translation to equivalent \LTLf formulas \citep{de2014insensitivity,camacho2019towards}. Here, we also provide the translation to their equivalent \PLTLf formula. Notably, many, but not all formulas, have a straightforward translation to the corresponding pure-future \LTLf formula.

%====================== PDDL3 TABLE ===========================
\begin{table}[]
\centering
\begin{tabular}{@{ }l@{ }l@{ }l@{ }}
\hline
\textsf{PDDL3} Operators & Equiv. \PLTLf Formula & Equiv. \LTLf Formula \\
\hline
\addlinespace
$(\textsf{at-end }\theta)$ & $\theta$ & $\Eventually(\theta \land \Last)$ \\
$(\textsf{always }\theta)$ & $\Historically \theta$ & $\Always \theta$ \\
$(\textsf{sometime }\theta)$ & $\Once \theta$ & $\Eventually \theta$ \\
$(\textsf{sometime-after }\theta_1\text{ }\theta_2)$ & $\lnot(\lnot\theta_2 \Since (\theta_1 \land\lnot\theta_2))$ & $\Always (\theta_1 \rightarrow \Eventually\theta_2) $ \\
$(\textsf{sometime-before }\theta_1\text{ }\theta_2)$ & $\Historically(\Once(a \land \Historically(a \lor \lnot b)) \rightarrow \Yesterday b)$ & $\theta_2 \Release \lnot \theta_1 $ \\
% $(\textsf{at-most-once }\theta)$ & $\Historically(\theta \rightarrow \Wyesterday(\Historically(\lnot \theta))$ & $\Always(\theta \rightarrow \Wnext(\Always(\lnot \theta)) $ \\
\addlinespace
\hline
\addlinespace
$(\textsf{hold-during }n_1\text{ }n_2\text{ }\theta)$ & 
    \begin{tabular}{@{}l@{}}$\bigvee_{0\leq i\leq n_1} (\theta \land \Yesterday^{i}(\start)) \lor$ \\ $\bigwedge_{n_1 < i\leq n_2} \Historically(\theta \lor \Wyesterday^{i}(\Yesterday(\true)))$\end{tabular}
      &     \begin{tabular}{@{}l@{}}$\bigvee_{0\leq i\leq n_1} \Next^{i}(\theta \land \Last) $ \\ $\lor \bigwedge_{n_1 < i\leq n_2} \Wnext^{i}(\theta)$\end{tabular} \\
\addlinespace
\hline
\addlinespace
$(\textsf{hold-after }n\text{ }\theta)$ & $\bigvee_{0\leq i\leq {n+1}} (\theta \land \Yesterday^{i}(\start))$ & $\bigvee_{0\leq i\leq {n+1}} \Next^{i} (\theta \land \Last)$ \\
\addlinespace
\hline
\end{tabular}
\caption{\textsf{PDDL3} operators with their equivalent \PLTLf and \LTLf formulas. Superscripts abbreviate nested temporal operators. $\theta$ is a propositional formula on fluents, $\Wyesterday\phi \equiv \lnot\Yesterday\lnot\phi$, $\phi_1 \Release \phi_2 \equiv \lnot(\lnot \phi_1 \lUntil \lnot\phi_2)$, and $\Last$ denotes the end of the trace in \LTLf.}
\label{table:pddl3}
\end{table}
%================================================================

%====================== DECLARE TABLE ===========================
\begin{table}[]
\centering
\begin{tabular}{@{ }l@{ }l@{ }l@{ }}
\hline
\declare Templates & Equiv. \PLTLf Formula & Equiv. \LTLf Formula \\
\hline
\addlinespace
$\textsf{init}(a)$ & $ \Once(a \land \lnot\Yesterday(\true))$ & $a$ \\
$\textsf{existence}(a)$ & $\Once a$ & $\Eventually a$ \\
$\textsf{absence}(a)$ & $\lnot\Once a$ & $\lnot \Eventually a$ \\
$\textsf{responded-existence}(a,b)$ & $\Once a \rightarrow \Once b$ & $\Eventually a \rightarrow \Eventually b$ \\
$\textsf{response}(a,b)$ & $\lnot(\lnot b \Since (a \land\lnot b))$ & $\Always (a \rightarrow \Eventually b) $ \\
$\textsf{precedence}(a,b)$ & $\Once(a \land \Historically(a \lor \lnot b))) \lor \Historically(\lnot b)$ & $(\lnot b \lUntil a) \lor \Always(\lnot b)$ \\
$\textsf{succession}(a,b)$ & \multicolumn{2}{c}{$\textsf{response}(a,b) \land \textsf{precendece}(a,b)$}  \\
$\textsf{chain-precedence}(a,b)$ & $\Historically(b \rightarrow \Yesterday a)$ & $\Always(\Next b \rightarrow a) \land \lnot b$ \\
\addlinespace
\hline
\addlinespace
$\textsf{chain-succession}(a,b)$ & 
\begin{tabular}{@{}l@{}}$(\Historically(\Yesterday a \rightarrow b) \land \lnot a) \land$ \\ $\Historically(\Yesterday(\lnot a) \rightarrow \lnot b)$\end{tabular}
& $\Always(a \leftrightarrow \Next b)$ \\
\addlinespace
\hline
\addlinespace
$\textsf{not-co-existence}(a,b)$ & $\Once a \rightarrow \lnot \Once b$ & $\Eventually a \rightarrow \lnot \Eventually b$ \\
$\textsf{not-succession}(a,b)$ & $\Historically (b \rightarrow \lnot \Once a)$ & $\Always (a \rightarrow \lnot \Eventually b)$ \\
$\textsf{not-chain-succession}(a,b)$ & $\Historically (b \rightarrow \lnot \Yesterday a)$ & $\Always (a \rightarrow \lnot \Next b)$ \\
$\textsf{choice}(a,b)$ & $\Once a \lor \Once b$ & $\Eventually a \lor \Eventually b$ \\
$\textsf{exclusive-choice}(a,b)$ & $(\Once a \lor \Once b) \land \lnot (\Once a \land \Once b)$ & $(\Eventually a \lor \Eventually b) \land \lnot (\Eventually a \land \Eventually b)$ \\
\addlinespace
\hline
\end{tabular}
\caption{\declare templates with their equivalent \PLTLf and \LTLf formulas. $a, b$ are atomic formulas.}
\label{table:declare}
\end{table}
%================================================================
Further examples of \PLTLf formulas can be found in the literature of various areas of AI. For instance, in \cite{Bacchus96} \PLTLf was used in to express non-Markovian rewards in decision processes, whereas in \cite{FisherW05,Gabaldon11,knobbout2016dynamic,AlechinaLD18} \PLTLf is used to express norms in multi-agent systems. 
%  temporally extended goals are considered quite interesting in the context of planning, and for a wide number of examples of \PLTLf objectives, one could, e.g., look at \cite{Bacchus96} (there \PLTLf was used in the context of MDPs). Also, \PLTLf is commonly employed to specify conditions and norms, as discussed in \cite{degiacomo2020pure}, and the references therein, see, e.g., \cite{FisherW05,Gabaldon11,knobbout2016dynamic,AlechinaLD18}.
% % various properties can be naturally expressed using \PLTLf as high-level specification \citep{FisherW05,knobbout2016dynamic, AlechinaLD18}. 

\subsection{Computational Advantage of \PLTLf over \LTLf}
\PLTLf has the same expressive power of \LTLf. However, compared to \LTLf, \PLTLf gives an exponential (worst-case) computational advantage in several contexts. 
Both \LTLf and \PLTLf can be \emph{translated} into  an equivalent \emph{Alternating Finite-state Automaton} (\AFA), in linear time. Here, equivalent means that if a formula is true in a trace, then the trace can be seen as a string recognized by the \AFA.
The \PLTLf computational advantage stems from a well-known language theoretic property of regular languages, for which the \AFA corresponding to the \PLTLf formula, can be transformed, in single exponential time, into a \DFA recognizing the reverse language \citep{CKS81}. Note that, in general, the \DFA for the language itself (not its reverse) can be double-exponentially larger than the \AFA. Hence, the conversion of \PLTLf formulas to their corresponding \DFAs is worst-case single exponential time (vs. double exponential time for \LTLf formulas) \citep{degiacomo2020pure}.
Consequently, the computational complexity of many problems involving temporal logics on finite traces, which explicitly or implicitly require to compute the corresponding \DFA, is affected by the exponential savings of \PLTLf.
For instance, this is the case for planning in non-deterministic domains (FOND)
\citep{camacho2017nondeterministic,degiacomo2018automata}, reactive synthesis
\citep{DegVa15,camacho2018finite}, MDPs with non-Markovian
rewards \citep{Bacchus96,BrafmanGP18rewards}, reinforcement learning
\citep{DeGiacomoIFP19,CamachoIKVM19}, and non-Markovian planning and
decision problems \citep{BrafmanG19planning,BrafmanG19regular}. 
Instead, note that this is not the case for planning in deterministic domains, where it is sufficient to reduce the temporal logic formula describing the goal into an \NFA \citep{degiacomo2018automata}.

Finally, we observe that, although there is often a computational advantage in using \PLTLf wrt \LTLf, transforming one into the other (and vice versa) can be triply exponential in the worst-case, and these are the best-known bounds \citep{degiacomo2020pure}. Therefore, the property of interest should be succinctly expressible directly in \PLTLf to exploit the computational advantage, as is often the case when the specifications \emph{naturally} talk about the past.
Certainly, if a property can be naturally expressed with \PLTLf, theoretical and practical evidence indicates that \PLTLf should be the language to go with.

\section{Planning for \PLTLftitle Goals}
\label{sec:planning-for-pltlf-goals}

In this paper, we consider both classical and \FOND planning for \PLTLf goals.
% In general, temporally extended goals are considered quite interesting in the context of planning, and for a wide number of examples of \PLTLf objectives, one could, e.g., look at \cite{Bacchus96} (there \PLTLf was used in the context of MDPs). Also, \PLTLf is commonly employed to specify conditions and norms, as discussed in \cite{degiacomo2020pure}, and the references therein, see e.g. \cite{Gabaldon11,FisherW05,knobbout2016dynamic,AlechinaLD18}.
%%
%%
% \subsection{Classical Planning for \PLTLf Goals}
% \todo[inline]{Introduce directly classical planning for \PLDLf goals}
% Following~\cite{geffner2013concise}, a \emph{Classical Planning} domain for temporally extended goals in deterministic domains
% 
% \subsection{FOND Planning for \PLTLf Goals}
% \todo[inline]{Then Introduce dicrectly classical planning for \PLDLf goals}
%%
Following \cite{geffner2013concise}, a planning domain model is a tuple $\D = \tup{2^{\F}, A, \alpha, tr}$, where $2^{\F}$ is the set of possible states and $\F$ is a set of fluents (atomic propositions); $A$ is the set of actions; $\alpha(s) \subseteq A$ represents the set of applicable actions in state $s$; and $tr(s, a)$ represents the non-empty set of successor states that follow action $a$ in state $s$. 
Such a domain model $\D$ is assumed to be compactly represented (e.g., in \PDDL~\citep{mcdermott1998pddl}), hence its size is $|\F|$.
Given the set of literals of $\F$ as $\mathit{Lits}(\F) := \F \cup \{ \lnot f \mid f \in \F \}$, every action $a \in A$ is usually characterized by $\tup{\Pre_a, \Eff_a}$, where $\Pre_a \subseteq \mathit{Lits}(\F)$ represents action preconditions and $\Eff_a$ represents action effects. An action $a$ can be applied in a state $s$ if the set of literals in $\Pre_a$ holds true in $s$. 

In \emph{Classical Planning}, the result of applying $a$ in $s$ is a successor state $s^\prime$ determined by $\Eff_a$ (i.e., actions have deterministic effects: $|tr(s, a)| = 1$ in all states $s$ in which $a$ is applicable). On the other hand, in \emph{Fully Observable Nondeterministic Domain} (\FOND) planning, the successor state $s^\prime$ is nondeterministically drawn from one of the $\Eff^{i}_{a}$ in $\Eff_a = \{ \Eff^{1}_{a}, \dots, \Eff^{n}_{a} \}$. That is, some action effects have an uncertain outcome and cannot be predicted in advance (i.e., $|tr(s, a)| \geq 1$ in all states $s$ in which $a$ is applicable). 
In \PDDL, the uncertain outcomes are expressed using the \pred{oneof}~\citep{bryce2008international} keyword, as widely used by several \FOND planners.
Intuitively, a nondeterministic domain evolves as follows: from a given state $s$, the agent chooses
a possible action $a\in A$, 
%(i.e., such that there exists an $s^\prime$ such that $(s, a, s^\prime) \in tr)$, 
after which the environment chooses a successor state $s^\prime$ such that $(s, a, s^\prime) \in tr$. 
In choosing its actions the agent can consider the whole trace (i.e., sequence of states) produced so far since the domain is fully observable. 

%The agent can choose its action based on the history of states so far (i.e., the agent has full observation).

% Differently from \FOND Planning where some actions are \textit{nondeterministic} (i.e, $|tr(s, a)| \geq 1$ in all states $s$ in which $a$ is applicable) and effects cannot be predicted in advance, in Classical Planning~\citep{ghallab2004automated} all actions have \textit{deterministic} effects (i.e., $|tr(s, a)| = 1$ in all states $s$ for which $a$ is applicable). 
% In this way, Classical Planning can be considered as a special case of \FOND Planning.

% Unlike \textit{Classical Planning}~\citep{ghallab2004automated}, in which actions are \textit{deterministic} (i.e., $|tr(s, a)| = 1$ in all states $s$ for which $a$ is applicable), in \FOND planning, some actions have \textit{uncertain outcomes}, in that they are \textit{nondeterministic} (i.e., $|tr(s, a)| \geq 1$ in all states $s$ in which $a$ is applicable), and effects cannot be predicted in advance. 
% \PDDL, expresses uncertain outcomes using the \texttt{oneof}~\citep{bryce2008international} keyword, as widely used by several \FOND planners, such as FOND-SAT~\citep{geffner2018compact}, MyND~\citep{mattmuller2010pattern}, and PRP~\citep{muise2012improved,MuiseAAAI14,MuiseICAPS14}. 

Formally, a planning problem for \PLTLf goals is defined as follows.

\begin{definition}\label{def:planning-problem}
A planning problem is a tuple $\Gamma = \tup{\D, s_{0}, \varphi} $, where $\D$ is a domain model, $s_{0}$ is the initial state, i.e., an initial assignment to fluents in $\F$, and $\varphi$ is a \PLTLf goal formula over $\F$.
\end{definition}
A solution to planning problem $\Gamma$, when the domain model $\D$ is deterministic, is a \emph{plan} $\pi = a_0\dots a_n$, which is a sequence of actions $a \in A$ such that, when executed, induces a finite \emph{trace} (i.e., a finite sequence of states) $s_{0},\dots, s_{n}$, where $s_{i+1} \in tr(s_i, a_i)$ and $a_i \in \alpha(s_i)$ for $i = 0,\dots, {n-1}$, which satisfies the \PLTLf goal formula $\varphi$, i.e., $(s_{0},\dots, s_{n}) \models\varphi$.
To solve  $\Gamma$ for \PLTLf goals, we can build the deterministic automata for the domain and the nondeterministic automaton for the goal formula, compute their product, and then check non-emptiness on the resulting automaton returning a plan, if exists \citep{DegVa13,degiacomo2018automata}.

\begin{theorem}[\citealp{DegVa13}]
Classical Planning for \PLTLf goals is PSPACE-complete in both the domain and the goal formula.
\end{theorem}

Instead, when $\Gamma$ is a \FOND planning problem, solutions to $\Gamma$, i.e., plans, are \emph{strategies} (or \emph{policies}).
A strategy is defined as a partial function $\pi: (2^\F)^{+} \to A$ mapping traces into applicable actions. Note that when the strategy needs only finite memory, then it can be represented as a  finite-state transducer, and this is the case for \LTLf and \PLTLf goals \citep{degiacomo2018automata}.
A strategy $\policy$ for $\Gamma$, starting from the initial state $\s_0$, induces a set of  \emph{generated executions} $\Lambda$, each of which is a possibly infinite trace $s_{0},s_{1},\dots$
%obtained by choosing some possible outcome of actions instructed by the strategy.
where $s_0$ is the initial state, $(s_i, a_i,s_{i+1})\in tr$, $a_i = \policy(s_{0},\dots, s_{i})$ (hence $\policy(s_{0},\dots, s_{i})$ is not undefined), and $a_i\in\alpha(s_i)$,  for $i = 1,2,\dots$.
If for a certain state $s_n$ we have that $\policy(s_{0},\dots, s_{n})$ is undefined, then the generated execution $s_{0},\dots, s_{n}$ is a finite trace.

As usual, we consider two kinds of solutions to FOND planning problems: strong solutions and strong-cyclic solutions~\citep{CimattiPRT03}. 
A strategy $\policy$ is a \emph{strong solution} to $\Gamma$ with \PLTLf goal $\varphi$, if every generated execution is a finite trace $\trace$ such that $\last(\trace)\models \varphi$. A  strategy is a \emph{strong-cyclic} solution to $\Gamma$ with \PLTLf goal $\varphi$, if every generated execution that is a \emph{stochastic fair trace} is also a finite trace such that $\last(\trace)\models \varphi$, cf. \cite{AminofGR20}.
When a strategy $\policy$ is a solution (either strong or strong-cyclic, depending on the kind of solution we are interested in), we say that $\policy$ is \emph{winning}.

% Intuitively, a strategy $\policy$ for $\Gamma$ fulfills the formula $\varphi$ if and only if the sequence of states generated by $\policy$, despite the nondeterministic effects of the environment, is accepted by $\A_\varphi$, i.e., the automaton associated to the goal formula $\varphi$.

\begin{theorem}[\citealp{degiacomo2020pure}]
\FOND Planning (strong or strong-cyclic) for \PLTLf goals is EXPTIME-complete in both the domain and the goal formula.
\end{theorem}
In general, to solve  \FOND planning for \PLTLf goals, one can build the deterministic automata for the domain and for the goal formula, compute their product, and finally solve a \DFA game (with or without stochastic fairness) on the resulting automaton~\citep{degiacomo2018automata}.
The EXPTIME-complete in the \PLTLf goal formula contrasts with the 2EXPTIME-complete result for \LTLf goals. The EXPTIME upper-bound in the size of the \PLTLf goal formula  can be obtained using the above-mentioned automata technique~\citep{degiacomo2018automata}, with the additional observation that computing the \DFA corresponding to the \PLTLf goal is EXPTIME, instead of 2EXPTIME as for \LTLf~\citep{degiacomo2020pure}.

\section{Theoretical Bases to Handle \PLTLftitle Goals}
\label{sec:progression}

% \todo[inline]{add $\sub(\varphi), \PA(\varphi)$ here.}
In this section, we develop the bases for our technique. First, we observe that any sequence of actions produces a trace on which \PLTLf formulas can be evaluated. Therefore, while the planning process goes on, sequences of actions are produced, traces are generated, and  over them \PLTLf goals can be evaluated. The difficulty is that evaluating \PLTLf formulas requires a trace, and searching through traces is quite demanding. Instead, our technique does not consider traces at all. In particular, it makes and exploits the following observations: \myi to evaluate the \PLTLf goal formula only the truth value of its subformulas is needed; \myii every \PLTLf formula can be put in a form where its evaluation depends on the current propositional evaluation and the evaluation of a key set of \PLTLf subformulas at the previous instant; \myiii one can recursively compute and keep the value of such a small set of formulas as additional propositional variables in the state of the planning domain.
% We now give details of the technique to evaluate \PLTLf formulas on finite traces based on formula progression, which will be the foundational building block for our novel planning approach, and prove its correctness.

% We introduce a technique to evaluate \PLTLf formulas on finite traces based on formula progression, which will be the foundational building block for our novel planning approach explained later in this paper. 
% We now give the detail of the technique and prove its correctness.

% We introduce a technique to evaluate \PLTLf formulas on finite traces based on formula progression, which will be the foundational building block for our novel planning approach explained later in this paper. 

We start by denoting with $\sub(\varphi)$ the set of all subformulas of $\varphi$~\citep{DegVa13}.
% obtained from the abstract syntax tree of $\varphi$~\citep{DegVa13}, and by $\PA(\varphi) \subseteq \sub(\varphi)$ the set of maximal proper subformula(s) of $\varphi$~\citep{LiRPZV19}. Therefore, for a \PLTLf formula $\varphi$, we have:
% \begin{itemize}
%     \item $\PA(\varphi) = \{\varphi\}$ if $\varphi$ is a literal or a temporal formula;
%     \item $\PA(\varphi) = \PA(\varphi_1) \cup \PA(\varphi_2)$ if $\varphi = (\varphi_1 \land \varphi_2)$ or $\varphi = (\varphi_1 \lor \varphi_2)$.
% \end{itemize}
For instance, if $\varphi = a \land \lnot \Yesterday(b \lor c)$, where $a, b, c$ are atomic, then $\sub(\varphi) = \{ a, b, c, (b \lor c), \lnot \Yesterday(b \lor c),  \Yesterday(b \lor c), a \land \Yesterday(b \lor c)\}$. %, whereas $\PA(\varphi) = \{ a, \Yesterday(b \lor c) \}$.

In general, modalities in \LTL, and therefore in \LTLf, can be decomposed into present and future components \citep{Emerson1990TemporalAM,Bacchus96}. Analogously, \PLTLf formulas can be decomposed into present and past components, by recursively applying the following transformation function $\pnf(\cdot)$:
% \begin{definition}
% \label{def:pnf}
% A \PLTLf formula $\varphi$ is in Previous Normal Form~(\PNF) if $\PA(\varphi)$ consists of literals and $\Yesterday$-formulas only and can be computed as follows:
\begin{itemize}
    \item $\pnf(a) = a$;
    \item $\pnf(\Yesterday\phi) = \Yesterday\phi$;
    \item $\pnf(\phi_1 \Since \phi_2) = \pnf(\phi_2) \lor (\pnf(\phi_1) \land \Yesterday(\phi_1\Since\phi_2))$.    
    \item $\pnf(\phi_1 \land \phi_2) = \pnf(\phi_1) \land \pnf(\phi_2)$;
    \item $\pnf(\lnot\phi) = \lnot \pnf(\phi)$;
    % \item $\pnf(\Once\varphi) = \pnf(\varphi) \lor \Yesterday(\Once\varphi)$;
    % \item $\pnf(\Historically\varphi) = \pnf(\varphi) \land \Yesterday(\Historically\varphi)$;
\end{itemize}
For convenience, we extend the definition of $\pnf(\cdot)$ to $\Once\phi$ and $\Historically\phi$ as follows: $\pnf(\Once\phi) = \pnf(\phi) \lor \Yesterday(\Once\phi)$; and $\pnf(\Historically\phi) = \pnf(\phi) \land \Yesterday(\Historically\phi)$.
%\end{definition}
We say that a formula resulting from the application of $\pnf(\cdot)$ is in Previous Normal Form~(\PNF). 
% \todo[inline]{R1: This sentence isn't fully clear: "Note that formulas in Previous Normal Form have proper temporal subformulas (i.e., subformulas whose main construct is a temporal operator) only of the form $Y\phi$." Aren't there also temporal subformulas of the form $\phi_1 S \phi_2$? I believe that what is meant here is probably that the latter formulas only appear in the scope of a Y modality, but wasn't sure what "main construct" meant, and don't see a definition.}
Note that formulas in \PNF have proper temporal subformulas (i.e., subformulas whose main construct is a temporal operator) appearing only in the scope of the $\Yesterday$ operator. 
Note also that the formulas of the form $\Yesterday\phi$ in $\pnf(\varphi)$ are such that $\phi\in\sub(\varphi)$.

% Intuitively, we can transform any \PLTLf formula in its \PNF version separating what is true at the \emph{current} instant and what must have been true at the \emph{previous} instant.

\begin{theorem}
\label{th:pnf}
Every \PLTLf formula $\varphi$ can be converted to its \PNF form $\pnf(\varphi)$ in linear-time in the size of the formula. Moreover, $\pnf(\varphi)$ is equivalent to $\varphi$.
\end{theorem}
\begin{proof}
Immediate from the definition of $\pnf(\cdot)$ and the semantics of  \PLTLf  formulas.
\end{proof}

% When we progress a pure-past temporal formula, we need to know which subformulas become true at each instant of time while consuming propositional interpretations in the trace, which, in our case, corresponds to a sequence of planning domain states. 
Now, we show that to evaluate a \PLTLf formula $\varphi$, we only need to keep track of the truth values of some key subformulas of $\varphi$. %, while consuming propositional interpretations in the history, i.e., the sequence of planning domain states traversed so far.
To do so, we introduce $\Sphi\subseteq\sub(\varphi)$ as the set of \emph{propositions} of the form $\quoted{\Yesterday\phi}$ containing:
\begin{itemize}
    \item $\quoted{\Yesterday\phi}$ for each subformula of $\varphi$ of the form $\Yesterday\phi$;
    \item $\quoted{\Yesterday(\phi_1 \Since \phi_2)}$ for each subformula of $\varphi$ of the form $\phi_1 \Since \phi_2$.
\end{itemize}

% To establish the truth values of each proposition in $\Sphi$, we define $\sigmaY{}$ as:
To keep track of the truth of  each proposition in $\Sphi$, we define a specific interpretation $\sigmaY{}$:
$$
\sigmaY{} : \Sphi \to \{\top, \bot\}
$$
Intuitively, given an instant $i$, $\sigmaY{i}$ tells us which propositions in $\Sphi$ are true at instant $i$. 
By suitably maintaining the value of propositions in $\Sphi$ in $\sigmaY{i}$, we can 
evaluate a \PLTLf formula just by using the propositional interpretation in the current instant $i$ and the truth value assigned at the previous instant by $\sigmaY{i}$ to the subformulas involving the $\Yesterday$ operator. 
% , we just need to consider the value of subformulas only in two two consecutive instants of time the current one $\s_i$ given by  $\trace$. 
% In particular, the idea is to compute the true assignments of the \emph{current} instant with the exclusive knowledge of true assignments in the \emph{previous} instant.
% In particular, one instant of time corresponds to the \emph{current} instant, whereas the other instant of time corresponds to the \emph{previous} instant. 
% For such a reason, we added a Y-subscript to $\sigmaY{}$ to highlight the fact that it refers to the previous instant.
%%
%%
%%
% \todo[inline]{R2: In Def. 2, items 3, 4, and 5, I would recommend to better separate object language and meta language. In particular, some of the disjunction, conjunction and negation symbols on the right-hand sides of those three items should be meta-language "and"s, "or"s, and "not"s.}
\begin{definition}
\label{def:val}
Let $\s_i$ be a propositional interpretation over $\P$, $\sigmaY{i}$ a propositional interpretation over $\Sphi$, and $\phi$ a \PLTLf subformula in $\sub(\varphi)$, we define the predicate $\val(\phi,\sigmaY{i}, \s_i)$, recursively as follows:
\begin{itemize}
    \item $\val(a, \sigmaY{i}, \s_i)\tiff \s_i\models a$;
    \item $\val(\Yesterday\phi', \sigmaY{i}, \s_i)  \tiff  \sigmaY{i}\models\quoted{\Yesterday\phi'}$;
    % \item $\begin{aligned}[t]
    % &\val(\varphi_1\Since\varphi_2, \sigmaY{i}, \s_i) \tiff\\
    % &\val(\varphi_2, \sigmaY{i}, \s_i)   \lor (\val(\varphi_1, \sigmaY{i}, \s_i) \land \sigmaY{i}\models \quoted{\Yesterday(\varphi_1\Since\varphi_2)});
    % \end{aligned}$
    \item $\val(\phi_1\Since\phi_2, \sigmaY{i}, \s_i) \tiff \val(\phi_2, \sigmaY{i}, \s_i)   \lor (\val(\phi_1, \sigmaY{i}, \s_i) \land \sigmaY{i}\models \quoted{\Yesterday(\phi_1\Since\phi_2)});$
    % \item $\val(\varphi_1 \land \varphi_2, \sigmaY{i-1}, \s_i) \tiff\\ 
    % ~\qquad\qquad\val(\varphi_1, \sigmaY{i-1}, \s_i) \land \val(\varphi_2, \sigmaY{i-1}, \s_i)$;
    % \item $\begin{aligned}[t]
    % \val(\varphi_1\, \land & \,\varphi_2, \sigmaY{i}, \s_i) \tiff\\ 
    % &\val(\varphi_1, \sigmaY{i}, \s_i) \land \val(\varphi_2, \sigmaY{i}, \s_i);
    % \end{aligned}$
    \item $\val(\phi_1\, \land \,\phi_2, \sigmaY{i}, \s_i) \tiff \val(\phi_1, \sigmaY{i}, \s_i) \land \val(\phi_2, \sigmaY{i}, \s_i);$
    \item $\val(\lnot\phi', \sigmaY{i}, \s_i) \tiff \lnot\val(\phi', \sigmaY{i}, \s_i)$.
\end{itemize}
\end{definition}
Intuitively, the $\val(\phi,\sigmaY{i}, \s_i)$ predicate allows us to determine what is the truth value of any \PLTLf formula $\phi \in\sub(\varphi)$ by reading a propositional interpretation $\s_i$ from trace $\trace$ and keeping track of the truth value of the subformulas of the form $\Yesterday\phi'\in\sub(\varphi)$ by means of $\sigmaY{i}$.

% Now, given an assignment $\sigmaY{i-1}$ at time-step $i-1$ and a propositional interpretation $\s_i$ at the current time-step, we define the progression function $\prog(\cdot)$ to compute, at every instant $i$, the next assignment of subformulas in $\Sphi$ as $\prog(\sigmaY{i-1}, \s_i) = \sigmaY{i}$. Note again that $\sigmaY{i-1}$ and $\s_i$ refer to two distinct but consecutive instants of time, i.e., $i{-}1$ and $i$. In particular, $\sigmaY{i-1}$ refers to the \emph{previously} known assignment to propositional formulas in $\Sphi$, whereas $\s_i$ refers to the \emph{current} instant of time, which is needed to compute the new assignment $\sigmaY{i}$.
% Next, we generalize the progression function to finite traces by defining $\prog(\sigmaY{i-1}, \s_i \trace) =
% \prog(\prog(\sigmaY{i-1}, \s_i), \trace)$ for every $\s_i, \trace$ over $\P$. 

Now, given a trace $\trace=\s_0\cdots\s_n$ over $\P$, we compute a corresponding trace $\tracePlus=\sigmaY{0}\cdots\sigmaY{n}$ over $\Sphi$ where:
\begin{itemize}
    \item $\sigmaY{0}$ is such that $\sigmaY{0}(\quoted{\Yesterday\phi})\doteq\bot$ for each  $\quoted{\Yesterday\phi}\in\Sphi$;
    \item $\sigmaY{i}$ is such that $\sigmaY{i}(\quoted{\Yesterday\phi})\doteq\val(\phi, \sigmaY{i{-}1}, \s_{i{-}1})$, for all $i$ with $0 < i \leq {n}$.
\end{itemize}
% \todo[inline]{the following is not true anymore!}
% Note that while $\trace$ has at least one element, $\tracePlus$ has at least two elements: $\sigmaY{1}$ corresponding to $\s_0$ and the extra initial element $\sigmaY{-1}$.

% The trace $\tracePlus$ denotes the right sequence of assignments on which we can evaluate our \PLTLf formula. Intuitively, $\tracePlus$ can also be seen as obtained by inductively applying the progression on the trace $\trace$ starting from $\sigmaY{-1}$.

First, we show that for traces of length $1$ the following result holds.

\begin{lemma}
\label{lem:ev}
Let $\varphi$ be \PLTLf formula over $\P$, $\phi\in\sub(\varphi)$ a subformula of $\varphi$, and $\trace = \s_0$ a trace over $\P$ of length 1. Then, $\s_0\models\phi \tiff \val(\phi, \sigmaY{0}, \s_0)$.
\end{lemma}
\begin{proof}.
By structural induction on the formula $\phi$.
\begin{itemize}
    \item $\phi=p$. By definition of $\val(\cdot)$, $\val(p, \sigmaY{0}, \s_0) \tiff \s_0 \models p$. 
    
    \item $\phi = \Yesterday \phi'$. By definition of $\sigmaY{0}$, $\sigmaY{0}(\quoted{\Yesterday \phi'}) = \bot$, and by the semantics, $\s_0\not\models\Yesterday\phi'$. Therefore, the thesis holds.
    
    \item $\phi = \phi_1 \Since \phi_2$. 
    $\val(\phi_1\Since\phi_2, \sigmaY{i}, \s_i) \tiff 
    \val(\phi_2, \sigmaY{i}, \s_i)   \lor (\val(\phi_1, \sigmaY{i}, \s_i) \land \sigmaY{i}\models \quoted{\Yesterday(\phi_1\Since\phi_2)})$.
    By definition of $\sigmaY{0}$, $\sigmaY{0}(\quoted{\Yesterday(\phi_1 \Since \phi_2})) = \bot$, hence the formula above simplifies to $\val(\phi_2, \sigmaY{i}, \s_i)$. On the other hand,  by the semantics, $\s_0\models\phi_1 \Since \phi_2 \tiff s_0\models \phi_2$. Hence, by induction the thesis holds.

    \item $\phi=\phi_1 \land \phi_2$ or $\phi=\lnot \phi'$. The thesis holds by structural induction.
\end{itemize}
\end{proof}

% Recalling that, in general, we use $\slast(\rho)$ to denote the \emph{second last} element of a trace $\rho$ and $\last(\rho)$ its \emph{last} element, we can give the following result.
% $\slast(\rho)$ to denote the \emph{second last} element of a trace $\rho$ and 
% Recalling that, in general, we use $\last(\trace)$ to denote the \emph{last} element of a trace $\trace$, we can give the following result.
Next, we extend the previous result to all traces of any length.
\begin{theorem}
\label{th:progression}
Let \PLTLf formula $\varphi$ over $\P$, $\phi\in\sub(\varphi)$ a subformula of $\varphi$, $\trace$ a trace over $\P$, and $\tracePlus$ the corresponding  trace   over $\Sphi$. Then
\[\trace \models \phi \tiff \val(\phi,\last(\tracePlus),\last(\trace)).\]
\end{theorem}
\begin{proof}
We prove the thesis by double induction on the length of the trace $\trace$ and on the structure of the formula $\phi$.
\begin{itemize}
    % \item Base case: $\tau = \epsilon$. Then, by Lemma~\ref{lem:ev} we have that $\epsilon\models\varphi \tiff \ev(\varphi)=\top$. 
    % Then, $\ev(\varphi)=\top \tiff
    % \last(\trace^+)\models\quoted{\varphi}$,
    % because 
    % $\last(\trace^+) = \last(\sigmaPhi{0}) = \sigmaPhi{0}$ 
    % and by definition of $\sigmaPhi{0}$. Therefore, the thesis holds.
    
    \item Base case: $\trace = \s_0$. By Lemma~\ref{lem:ev}, the thesis holds.
    
    \item Inductive step: Let $\trace=\traceInd$. By inductive hypothesis, the thesis holds for the trace $\trace_{n-1}$ of length $n-1$: 
    $$
    \trace_{n-1} \models \phi \tiff \val(\phi, \last(\trace_{n-1}^\Plus), \last(\trace_{n-1})) 
    % \last(\s_{n-1}^+)\models\quoted{\varphi}
    % , \forall \psi\in\cl(\varphi)
    $$
    Now, we prove that the thesis holds also for $\traceInd$:
    % , with $\sigma\in\Sp$: 
    $$
    \traceInd \models \phi \tiff \val(\phi, \last((\traceInd)^\Plus), \last(\traceInd))
    % \last((\trace\sigma)^+)\models\quoted{\varphi}
    % , \forall \psi\in\cl(\varphi)
    $$
    
    % Let $s=\prog(s_0,\trace)$.
    To prove the claim, we now proceed by structural induction on the formula, knowing that $\last((\traceInd)^\Plus) = \sigmaY{n}$ and $\last(\traceInd) = \s_n$:
    
    \begin{itemize}
        \item $\phi=p$. We have that $\traceInd \models p \tiff \s_n\models p$. For the $\val(\cdot)$ predicate we have that $\s_n\models p \tiff \val(p, \sigmaY{n}, \s_n)$.
        Therefore, the thesis holds.
        
        \item $\phi=\Yesterday\phi'$. We have that $\last(\traceInd)\models \Yesterday\phi' \tiff \last(\trace_{n-1})\models \phi'$.
        By inductive hypothesis, $\trace_{n-1} \models \phi' \tiff \val(\phi', \last(\trace_{n-1}^\Plus), \last(\trace_{n-1}))$. For the $\val(\cdot)$ predicate $\val(\Yesterday\phi',\sigma_n,s_n) \tiff \sigma_n \models \quoted{\Yesterday\phi'}$, which in turn is defined as $\val(\phi',\last(\trace_{n-1}^\Plus), \last(\trace_{n{-}1}))$. Hence the thesis holds.

        \item $\phi=\phi_1\Since \phi_2$. In this case it suffices to remember that $s_n\models \phi_1\Since \phi_2$ iff $s_n\models \phi_2 \lor (\phi_1 \land \Yesterday(\phi_1\Since\phi_2))$.  
        On the other hand, $\val(\phi_1\Since\phi_2, \sigmaY{n}, \s_n)$ iff $\val(\phi_2, \sigmaY{n}, \s_n) \lor (\val(\phi_1, \sigmaY{n}, \s_n) \land \sigmaY{n}\models \quoted{\Yesterday(\phi_1\Since\phi_2)})$.
        By structural induction we have that $s_n\models \phi_1 \tiff \val(\phi_1, \sigmaY{n}, \s_n)$, and
        $s_n\models \phi_2 \tiff \val(\phi_2, \sigmaY{n}, \s_n)$. Moreover $s_n\models \Yesterday(\phi_1\Since\phi_2)$ iff $\last(\trace_{n{-}1})\models \phi_1\Since\phi_2$, and 
        $\sigma_n\models \quoted{\Yesterday(\phi_1\Since\phi_2)}$ iff 
        $\val(\phi_1\Since\phi_2,\last(\trace_{n{-}1}^\Plus),\last(\trace_{n{-}1}))$. Finally, $\last(\trace_{n{-}1})\models \phi_1\Since\phi_2$ iff  $\val(\phi_1\Since\phi_2,\last(\trace_{n{-}1}^\Plus),\last(\trace_{n{-}1}))$ holds by induction on the length of the trace.

    \item $\phi=\phi_1 \land \phi_2$ or $\phi=\lnot \phi'$. The thesis holds by structural induction.

    \end{itemize}
\end{itemize}
\end{proof}

\noindent
Theorem~\ref{th:progression} gives us the bases of our technique. Specifically, it guarantees that by keeping suitably updated $\sigmaY{}$, we can evaluate our \PLTLf goal only using the propositional interpretation in the current instant and the truth value of the (quoted) yesterday formulas in $\sigmaY{}$, instead of considering the entire trace.
\section{Handling \PLTLftitle Goals in PDDL}
\label{sec:pddl}

In this section, we exploit Theorem~\ref{th:progression} above to devise a new approach for classical and \FOND planning for \PLTLf goals.
% , that we call \emph{progression-based}. 
The key idea behind our approach is that, given a \PLTLf formula and a planning domain, instead of computing the  automaton for the \PLTLf goal $\varphi$ and then building the cross-product between such an automaton and the automaton corresponding to the domain, as done, e.g., in \cite{baier2006planning,torres2015polynomial,camacho2017nondeterministic,camacho2018finite,degiacomo2018automata}, we simply keep track of the values of the formulas in $\sigmaY{}$ during the search process for a plan/strategy.
% This would allow us to potentially save a lot of computational resources, as in most cases only part of the goal automaton is necessary to find a solution for a given  planning task.

% to explore its relevant part by means of formula progression. 
% This would allow us to potentially save a lot of computational resources, as in most cases only a subset of the goal \DFA is required to find a solution for an instance of the \FOND planning task.
We present a compilation of \PLTLf goal formulas in \PDDL that works for both classical and \FOND planning (with and without stochastic fairness). Hence, in this section  we generically refer to  planning problems,  possibly with nondeterministic actions effects.

In the planning literature, e.g., \citep{baier2006planning,torres2015polynomial,camacho2017nondeterministic,camacho-strong-19}, solving planning for temporally extended goals is done in three steps. The first step consists in the compilation of the original planning problem $\Gamma$ involving the temporally extended goal into a planning problem $\Gamma^\prime$ for standard reachability goals. Step two concerns the invocation of a sound and complete planner, as, e.g., \fastdownward~\citep{helmertFD} and \mynd~\citep{mattmuller2010pattern}, to compute a plan/strategy solving the compiled problem $\Gamma^\prime$. Finally, in the third step, the computed plan/strategy is reworked (in a polynomial way) to get the solution for the original problem $\Gamma$. The  advantage of such an approach is that once temporal goals have been compiled away, one can leverage any off-the-shelf planner to actually solve the task.
Here, we follow a similar process. However, instead of encoding the dynamics of the automata corresponding to the temporally extended goals into \PDDL, as in the aforementioned works, we exploit Theorem~\ref{th:progression} to do the compilation in the first step. Furthermore, we will not introduce any extra control action, thus our step three trivializes. 
% Indeed, the framework in Section~\ref{sec:progression} ensures that, at a given instant of time $i$, we only need to know the truth assignments of subformulas in $\Sphi$ holding at the previous instant ${i-1}$ to check the truth of formula. Thus, in our compilation, we can avoid the use of auxiliary actions related to the temporal goal formula.

Given a planning problem $\Gamma = \tup{\D, s_0, \varphi}$, where $\D=\tup{2^{\F}, A, \alpha, tr}$ is a planning domain, $s_0$ the initial state and $\varphi$ a \PLTLf goal, the compiled planning problem is $\Gamma^\prime = \tup{\D^\prime, s^\prime_0, G^\prime}$, where $\D^\prime=\tup{2^{\F^\prime}, A, \alpha^\prime, tr^\prime}$ is compiled planning domain, $s^\prime_0$ the new initial state and $G^\prime$ is new  reachability goal. 
% Intuitively, auxiliary parts of $\Gamma^\prime$ are used to synchronize the domain and the values of subformulas in $\Sphi$. As previously mentioned, we do not alternate between a \emph{world} phase and a \emph{synchronization} phase, as specified in previous encodings~\citep{baier2006planning,torres2015polynomial,camacho2017nondeterministic,camacho-strong-19}.
% between the domain's actions, which we call the \emph{acting phase}, and a system synchronization action that simulates formula progression, which we call the \emph{synchronizing phase}.
%%
Specifically, $\Gamma^\prime$ is composed by the following components.
\paragraph{Fluents}
$\F^\prime$ contains the same fluents of $\F$ and it is augmented with one fluent for each proposition $\quoted{\Yesterday\phi}$ in $\Sphi$ to keep track of propositional interpretations $\sigmaY{i}$.
%\begin{itemize}
    % \item one fluent for each subformula $\phi$ in $\sub(\varphi)$ to represent the $\val(\phi, \sigmaY{}, \trace_i)$ predicates of Definition~\ref{def:val};
%     \item one fluent for each $\val(\phi, \sigmaY{i-1}, \trace_i)$ predicate of Definition~\ref{def:val}. The total number of these predicates is equal to the number of subformulas $\phi$ in $\sub(\varphi)$.
%     % \item and one additional fluent, called \pred{act}, to switch between the acting phase and the synchronizing phase.
% \end{itemize}
% Formally, $\F^\prime = \F \cup \{ \quoted{\phi} \mid \quoted{\phi} \in \Sphi \} \cup \{ \val(\phi, \sigmaY{i-1}, \trace_i) \mid \phi \in \sub(\varphi) \} \cup \{\pred{act}\}$.
Formally, $\F^\prime = \F \cup \{ \quoted{\Yesterday\phi} \mid \quoted{\Yesterday\phi} \in \Sphi \}$.
%\cup \{ \val(\phi, \sigmaY{i-1}, \trace_i) \mid \phi \in \sub(\varphi) \}$.

\paragraph{Initial State}
The initial state is the same of the original problem $\Gamma$ for the original fluents in $\F$, whereas the new fluents $\quoted{\Yesterday\phi} \in \Sphi$ are assigned to the truth value given by $\sigmaY{0}$. That is $s'_0 = (\sigmaY{0},s_0)$.

\paragraph{Derived Predicates}
We make use of \emph{derived predicates} (aka axioms)~\citep{hoffmann2005deterministic}, which are nowadays natively supported by most state-of-the-art planners.
In particular, we include a derived predicate $\val_\phi$ for every subformula $\phi\in \sub(\varphi)$. 
These predicates are intended to be such that the current state $(\sigmaY{i},\s_i)\models \val_\phi$ iff $\val(\phi,\sigmaY{i},s_i)$. To do so, mimicking the rules in Definition~\ref{def:val}, we define the following derivation rules:

\begin{itemize}
    \item $\val_{a}\gets a$;
    \item $\val_{\Yesterday\phi} \gets \quoted{\Yesterday\phi}$;
    \item $\val_{\phi_1\Since\phi_2} \gets (\val_{\phi_2}  \lor (\val_{\phi_1} \land \quoted{\Yesterday(\phi_1\Since\phi_2)}))$;
    \item $\val_{\phi_1 \land \phi_2} \gets (\val_{\phi_1} \land \val_{\phi_2})$;
    \item $\val_{\lnot\phi} \gets \lnot\val_{\phi}$.
\end{itemize}
It is immediate to see that indeed we have that $(\sigmaY{i},s_i)\models \val_\phi$ iff $\val(\phi,\sigmaY{i},s_i)$.

% To compute the new true assignments $\sigmaY{}$ at each planning state, we need a way to encode $\val(\phi, \sigmaY{i-1}, \trace_i)$ predicate rules, described in Definition~\ref{def:val}, in \PDDL. If we see every rule as of the form $\val(\phi, \sigmaY{i-1}, \trace_i) \gets \mathsf{cond}(\phi)$, with $\mathsf{cond}(\phi)$ being the condition for $\val(\phi, \sigmaY{i-1}, \trace_i)$ to evaluate to true, then a clever encoding of all such rules defined for every subformula $\phi\in\sub(\varphi)$ is to use \emph{derived predicates} (aka axioms)~\citep{hoffmann2005deterministic}, which are nowadays natively supported by most state-of-the-art planners.
% The condition of every predicate $\val(\phi, \sigmaY{}, \trace_i)$'s rule is encoded as a \emph{derived predicate}, which is a natively supported \PDDL feature by most state-of-the-art planners. 
As a result, the set of  derived predicates in $\Gamma^\prime$, denoted as $\F^\prime_{der}$, comprises the set of derived predicates $\F_{der}$ in the original problem $\Gamma$ plus a new derived predicates $\val_\phi$ for every subformula $\phi$ in $\sub(\varphi)$, i.e.,  $\F^\prime_{der} = \F_{der}\cup \{ \val_{\phi}\mid \phi \in \sub(\varphi) \}$.
% Along with these new derived predicates, we have the corresponding derivation rules described above.

We highlight that the use of derived predicates allows us to elegantly model the mathematics of Section~\ref{sec:progression} (i.e., the $\val(\phi, \sigmaY{i}, \s_i)$) and are often convenient when dealing with more sophisticated forms of planning (see, e.g., \cite{borgwardt2022expressivity}). 
They also simplify the action schema and the goal descriptions, without introducing control predicates among the fluents, and hence without affecting the search too much, as shown in~\cite{thiebaux2005defense}. 

\paragraph{Domain Actions}
Every domain's action in $A$ is modified on its effects by adding a way to update the assignments of propositions in $\Sphi$. The update of assignments can be modeled by a set of conditional effects (for each $\quoted{\Yesterday\phi} \in \Sphi$) of the form:
\[\begin{array}{lcl}
\val_\phi &\to& \quoted{\Yesterday\phi}\\
\lnot \val_\phi &\to& \lnot\quoted{\Yesterday\phi}
\end{array}\]
Note that these effects are exactly the same for every action $a\in A$. 
Also, since $\sigmaY{i}$ maintains values of $\quoted{\Yesterday\phi}$ in $\Sphi$ they are independent of the effect of the action on the original fluents, which, instead, is maintained in the propositional interpretation $\s_i$.
This means that we can compute the next value of $\sigmaY{}$ without knowing neither which action has been executed nor which effect such an action has had on the original fluents. 

Formally, let 
$
\Val = \{ \val_{\phi} \to \quoted{\Yesterday\phi}, \lnot \val_{\phi} \to \lnot \quoted{\Yesterday\phi}\mid \quoted{\Yesterday\phi} \in \Sphi \}.
$
% This set of conditional effects $\Val$ is computed just once for all actions $a\in A$, and is added to every action $a$ in conjunction with the original effects of $a$. Intuitively, every domain action will have its own effects plus a set of conditional effects referring to the \PLTLf goal.
The set of actions $A$ in $\Gamma'$ remains the same, as in the original problem $\Gamma$. For all $a \in A$, we have that the precondition in $\Gamma'$ are $\Pre'_{a} = \Pre_a$ and the effect in $\Gamma'$ are $\Eff'_{a} = \Eff_a \cup~\Val$.
Note that, the auxiliary part $\Val$ in $\Eff'_{a}$ \emph{deterministically} updates subformulas values in $\Sphi$, without affecting any fluent $f\in\F$ of the original domain model.
This is crucial to the encoding correctness.

% \paragraph{Progression Action}
% During the synchronizing phase the only available executable action is the progression action, which we call $\pred{prog}$. Such an action allows us to progress the assignments of propositions in $\Sphi$ at each instant of time. The precondition of $\pred{prog}$ is simply $\Pre_\pred{prog} = \{ \lnot \pred{act} \}$, while its effects are defined by conditional effects of the form $\mathsf{cond}(\phi) \to \mathsf{eff}(\phi)$ as follows:
% $$
% \Eff_\pred{prog} = \{ \val(\phi, \sigmaY{i-1},\trace_i) \to \quoted{\phi} \mid \quoted{\phi} \in \Sphi \} \cup \{ \pred{act}\}
% $$
% Here, note that the auxiliary action $\pred{prog}$ is deterministic and does not affect any fluent $f\in\F$ of the original domain model.

\paragraph{Goal}
The  goal in $\Gamma^\prime$ is specified as $G^\prime = \{ \val_{\varphi}\}$. That is, we want that the $\val(\varphi, \sigmaY{n},\s_n)$, corresponding to the original \PLTLf goal formula $\varphi$, holds true at the last instant, so as to exploit Theorem~\ref{th:progression}.
% , and, at the same time, that $\sigmaY{}$ is updated to evaluate the $\val(\varphi, \sigmaY{n-1},\trace_n)$ predicate on the right formula assignments.

It is easy to see that our compilation is polynomially related to the original problem.
\begin{theorem}
The size of the compiled planning problem $\Gamma^\prime$ is polynomial in the size of the original problem $\Gamma$. In particular, the additional fluents introduced are linear in the size of the temporally extended \PLTLf goal $\varphi$ of $\Gamma$ (in fact, in the number of $\varphi$'s subformulas of the form $\Yesterday\phi$ and $\phi_1\Since \phi_2$). 
% The \PDDL compilation introduced above introduces an overhead wrt the original \PDDL which is polynomial in the size of the goal formula $\varphi$.
\end{theorem}
% \todo[inline]{check proof}
\begin{proof}
Immediate, by analyzing the construction.
% The size of the new planning problem $\Gamma^\prime$ is given by $\mid \F^\prime \mid = \mid \F \cup \sub(\varphi) \cup \Sphi \mid$. 
% % And, all components referring to the goal formula $\varphi$ are polynomial in the size of the length of the formula itself.
% Both $\sub(\varphi)$ and $\Sphi$ are of size polynomial in the size of the length of the formula $\varphi$ (in fact, linear).
\end{proof}

Next, we turn to correctness.
Let  $\Gamma=(\D,s_0,\varphi)$ be a planning problem, where $\D$ is a (possibly nondeterministic) domain, $s_0$ is the initial state, and $\varphi$ is  a \PLTLf goal formula, and let $\Gamma'=(\D',s'_0,G')$ be the corresponding compiled planning problem as previously defined.

Any trace $\trace'=s'_0,\ldots, s'_n$ on $\D'$ can be seen as $\trace'=zip(\tracePlus,\trace)$,  where
$\trace = \s_0,\dots,\s_n \in (2^\F)^+$, $\tracePlus=\sigmaY{0},\ldots,\sigmaY{n} \in (2^\Sphi)^+$,
where each element of $\trace^{\prime}$ is of the form $\s'_i= (\sigma_{i}, \s_i)$ for all $i \geq 0$.
Given a trace $\trace'=s'_0,\ldots, s'_n$ on the compiled planning domain $\D'$, there exist a single trace $\trace'\mid_{\F}=\trace = \s_0,\dots,\s_n$ on the original planning domain $\D$. 
Conversely, given a trace $\trace = \s_0,\dots,\s_n$ on the original planning domain $\D$, there exists a unique corresponding trace $\tracePlus$, and hence a single $\trace'= zip(\tracePlus,\trace)$ on the compiled domain $\D'$. 

For every strategy $\policy: (2^\F)^+ \to A$ for the planning problem $\Gamma$ with \PLTLf goal $\varphi$, we can build the strategy $\policy^\prime: (2^{\F^\prime})^+ \to A$ for $\Gamma^\prime$ as follows:
\[\begin{array}{lcl}
\policy^\prime(\trace') = a &\tiff& \policy(\trace'\mid_\F) =a \\
\policy^\prime(\trace') \mbox{ is undefined} &\tiff&  \policy(\trace'\mid_\F) \mbox{ is undefined}.\\
\end{array}
\]
\begin{lemma}\label{th:original2compiled}
If $\policy: (2^\F)^+ \to A$ is a winning strategy for the \FOND planning problem $\Gamma$ with \PLTLf goal $\varphi$, then $\policy^\prime: (2^{\F^\prime})^+ \to A$, defined as above, is a winning strategy for compiled planning problem $\Gamma^\prime$.
\end{lemma}
\begin{proof}
Strategy $\policy$ is winning if every generated execution $\trace$ (that is stochastic fair, for strong-cyclic solutions) is finite, i.e., $\policy(\trace)$ is undefined, and  such that $\last(\trace)\models \varphi$.
Correspondingly, the strategy $\policy^\prime$ induces the finite generated execution $\trace'=zip(\tracePlus,\trace)$.
Then, $\val(\varphi,\last(\tracePlus),\last(\trace))$ holds by Theorem~\ref{th:progression}, so we have that
$\last(\trace')\models \val_\varphi$. 
On the other hand, if a generated execution $\trace'$ is finite, i.e., such that $\policy^\prime(\trace')$ is undefined, then $\policy$ induces a corresponding finite generated execution $\trace = \trace'\mid_\F$. Being $\policy$ winning, it must be the case that $\last(\trace)\models \varphi$. Hence, by Theorem~\ref{th:progression}, $\last(\trace')\models \val_\varphi$.
Thus, if $\policy$ is winning for $\Gamma$, then $\policy^\prime$ is winning for $\Gamma^\prime$.
\end{proof}

Now we consider the converse.
For every strategy $\policy^\prime: (2^{\F'})^+ \to A$ for the compiled planning problem $\Gamma'$, we can build the strategy $\policy: (2^{\F})^+ \to A$ for the original problem $\Gamma$ with \PLTLf goal $\varphi$ as follows (where $\trace^\prime=zip(\tracePlus,\trace)$):
\[\begin{array}{lcl}
\policy(\trace) = a &\tiff& \policy^\prime(\trace') = a \\
\policy(\trace) \mbox{ is undefined} &\tiff& \policy^\prime(\trace') \mbox{ is undefined}.
\end{array}
\]
\begin{lemma}\label{th:compiled2original}
If $\policy^\prime: (2^{\F^\prime})^+ \to A$ is a winning strategy for compiled planning problem $\Gamma^\prime$, then $\policy: (2^\F)^+ \to A$, defined as above, is a winning strategy for the \FOND planning problem $\Gamma$ with \PLTLf goal $\varphi$.
\end{lemma}
\begin{proof}
Strategy $\policy^\prime$ is winning if every generated execution $\trace'$ (that is stochastic fair, for strong-cyclic solutions) is finite, i.e., such that $\policy^\prime(\trace')$ is undefined, and such that $\last(\trace')\models \val_\varphi$.
Correspondingly, the strategy $\policy$ induces the finite generated execution $\trace = \trace'\mid_\F$. 
Then, by Theorem~\ref{th:progression}, considering that $\val(\varphi,\last(\tracePlus),\last(\trace))$ holds, we have that $\last(\trace)\models\varphi$. 
On the other hand, if a generated execution $\trace$ is finite, i.e., such that $\policy(\trace)$ is undefined then 
$\policy^\prime$ induces a corresponding finite generated execution $\trace'=zip(\tracePlus,\trace)$.
Being $\policy^\prime$ winning, we have that 
$\last(\trace')\models \val_\varphi$. Hence, by Theorem~\ref{th:progression}, $\last(\trace)\models\varphi$.
Thus, if $\policy^\prime$ is winning for $\Gamma^\prime$, then $\policy$ is winning for $\Gamma$.
\end{proof}

By exploiting Lemma~\ref{th:original2compiled} and Lemma~\ref{th:compiled2original}, we can show the correctness of our technique.
\begin{theorem}[Correctness]
\label{th:correctness}
Let $\Gamma$ be a (classical, \FOND strong or \FOND strong-cyclic) planning problem with a \PLTLf goal $\varphi$, and $\Gamma^\prime$ be the corresponding compiled (classical, \FOND strong or \FOND strong-cyclic, resp.) planning problem with reachability goal $G^\prime$. Then, $\Gamma$ has a winning strategy $\policy: (2^\F)^+ \to A$ iff $\Gamma^\prime$ has a winning strategy $\policy^\prime: (2{^\F}^\prime)^+ \to A$.
\end{theorem}
\begin{proof}
Direct consequences of  Lemma~\ref{th:original2compiled} and Lemma~\ref{th:compiled2original}.
\end{proof}

\begin{corollary}
Let $\Gamma$ be a (classical, \FOND strong or \FOND strong-cyclic) planning problem with a \PLTLf goal $\varphi$, and $\Gamma^\prime$ be the corresponding compiled (classical, \FOND strong or \FOND strong-cyclic, resp.) planning problem with reachability goal $G^\prime$. Then, every sound and complete planner (classical, \FOND strong or \FOND strong-cyclic, resp.) returns a winning strategy $\policy^\prime$ for $\Gamma^\prime$ if a winning strategy $\policy$ for $\Gamma$ exists. If no solution exists for $\Gamma^\prime$, then there is no solution for $\Gamma$.
\end{corollary}
\begin{proof}
If a winning strategy $\policy^\prime$ for $\Gamma^\prime$ exists, then for Theorem~\ref{th:correctness} there must be a winning strategy $\policy$ for $\Gamma$.
For the second part, suppose a sound and complete planner returns no solution for $\Gamma^\prime$, but a solution $\policy$ for $\Gamma$ does exist. This means that there will be at least an execution of $\policy$ satisfying the \PLTLf goal formula $\varphi$. However, for the completeness of the planner and for Theorem~\ref{th:correctness} there must be a corresponding winning strategy $\policy^\prime$, which contradicts the hypothesis. Therefore, the thesis holds.
\end{proof}

It is important to observe that strategies returned by a FOND planner for $\Gamma'$ are going to be ``memory-less'' policies of the form $\Pi'(s')=a$ or $\Pi'(s')$ undefined at the goal. These can be immediately transformed in trace-based strategies by defining:
\[
\begin{array}{lcl}
\policy(\trace')= a &\tiff& \Pi'(\last(\trace')) = a\\
\policy(\trace') \mbox{ is undefined} &\tiff& \Pi'(\last(\trace')) \mbox{ is undefined}.
\end{array}
\]
This possibility is essential, because the strategies for the original problem $\Gamma$ with a \PLTLf goal $\varphi$ cannot be memory-less policies, but must be memory-full strategies. In other words, they need to be finite-state controllers or transducers.
We can use the component $\sigmaY{i}$ of $s'_i=(\sigmaY{i},s_i)$  as the state of the transducer, $\sigmaY{i+1}(\quoted{\Yesterday\phi}) = \val(\phi,\sigmaY{i},s_i)$ (for each $\quoted{\Yesterday\phi}\in \Sphi$) as the factorized transition function, and $\Pi'(s'_i)$ as the output function of the transducer.

In the case of deterministic domains, we do not need these general forms of strategies and sequences of actions suffice, and these are in direct correspondence between the two domains.

% \begin{theorem}[Soundness]
% Let $\Gamma$ be a \FOND planning problem with a \PLTLf goal $\varphi$, and $\Gamma^\prime$ be the corresponding compiled \FOND planning problem with a reachability goal state. Then, solutions to $\Gamma^\prime$ are solutions to $\Gamma$.
% \end{theorem}
% \begin{proof}

% \end{proof}
% \begin{theorem}[Completeness]
% Let $\Gamma$ be a \FOND planning problem with a \PLTLf goal $\varphi$, and $\Gamma^\prime$ be the corresponding compiled \FOND planning problem with a reachability goal state. Then, if a solution exists for $\Gamma$, then one exists for $\Gamma^\prime$.
% \end{theorem}
% \begin{proof}

% \end{proof}

% \todo[inline]{To prove the correctness of our approach, we can exploit DFAs. In particular, we can show that our approach is the same as the Cartesian product between the DFA of the domain and the DFA of the formula (as per DeGiacomoRubin2018). The difference is in the DFA construction: on the one hand we have a complete minimal DFA; on the other, we build only part of the entire DFA -- the specific part that we are interested in -- but both DFAs recognize the same set of traces, which guarantees that found plans are correct.} 
\section{Experiments}
\label{sec:experiments}

We implemented the approach presented in Section \ref{sec:pddl} in a tool called \planforpast (\planforpastshort)\footnote{The tool is available online at \url{https://github.com/whitemech/planning-for-past-temporal-goals}}.
\planforpast takes in input a \PDDL domain, a \PDDL problem and a \PLTLf formula, and gives as output a compiled version of the \PDDL domain and the \PDDL problem.
Then, the planning task can be delegated to a state-of-the-art (SOTA) planner. In our experiments, we considered \fastdownward (\fastdownwardshort) \citep{helmertFD} and \mynd \citep{mattmuller2010pattern} as representative SOTA planners, which are sound and complete, for deterministic and nondeterministic domains, respectively. Combined with our compilation tool, they give the planners \planforpastfd (\planforpastfdshort for short) and \planforpastmynd (\planforpastmyndshort for short), respectively.
We used A$^*$ as algorithm for \fastdownward and LAO$^*$ as algorithm for \mynd. On both we adopt the FF heuristic. 

%\todo[inline]{revise}
% We preferred \mynd over \prp \citep{muise2012improved}, another well-known SOTA \FOND planner, because the latter is not compatible with some aspects of our implementation. In particular, \mynd natively supports the \PDDL derived predicates. We could have compiled away such derived predicates into additional actions and predicates and then use \prp; however, such a compilation can be quite expensive, as pointed out in \citep{thiebaux2005defense}.
We chose \mynd because it  natively supports \PDDL derived predicates and disjunctions in conditional effects. 
Note, however, that  derived predicates and disjunctions in conditional effects can be compiled away into additional actions and predicates, though with some overhead \citep{thiebaux2005defense}. This allows for applying our techniques to other SOTA \FOND planners like  \prp \citep{muise2012improved}.

We evaluated our compilation tool against existing compilation tools for temporally extended goals, by comparing 
% several
metrics of the performances of the planners over the compiled domains and problems.
 
% \noindent\textbf{Metrics.} The tools are compared on (1) runtime and (2) the number of expanded search nodes.

\paragraph{Baselines} We use the following baselines: 
\fondforlp (\fondforlpshort for short) \citep{fond4ltlf} and \ltlfondtofond (\ltlfondtofondshort for short) \citep{camacho2017nondeterministic,camacho-strong-19},
which are two compilers for FOND planning problems for temporally extended goals. 
\fondforlpshort supports both \LTLf and \PLTLf goals, whereas \ltlfondtofondshort only \LTLf goals. Both tools are combined with \fastdownwardshort and \myndshort, along with \planforpastshort, to do planning for temporal goals.
% \todo[inline]{la frase che segue si ripete sopra}
% As planners, we used \fastdownwardshort and \myndshort for the deterministic and nondeterministic parts of the experiments, respectively.

\paragraph{Experiment Setup} Experiments were run on a cloud-managed virtual machine, endowed with an Intel-Xeon processor running at 2.2 GHz, with 4GB of memory and 300 seconds of time limit. The correctness of \planforpastshort was also empirically verified by
comparing the results with those from all baseline tools. No
inconsistencies were encountered for all solved instances.

\paragraph{Experiment Types}
We run two types of experiments, both on deterministic and nondeterministic domains.
\begin{enumerate}
    \item \label{exp-overhead} \emph{Overhead}. In this experiment, we aim to discover the overhead introduced by our compilation technique when using the same goal of the planning problem, considered as the reachability goal $\Once(goal)$ expressed in \PLTLf. We compare the execution with a SOTA planner and the execution with the compiled version of the problem.
    \item \label{exp-comparison} \emph{Scalability over \PLTLf goals}. In this experiment, our aim is to measure the scalability of the planners over the compiled domain and problem computed by \planforpastshort, and compare it with the other compilation techniques, i.e. \fondforlpshort and \ltlfondtofondshort, using the same \PLTLf goal across planners. 
    In particular, we performed two variants:
    \begin{enumerate}
        \item[\mylabel{exp-comparison-a}{(2a)}] \label{exp-comparison-a} We fix the size of the \PLTLf goal, while scaling the size of the problem;
        \item[\mylabel{exp-comparison-b}{(2b)}]\label{exp-comparison-b} We scale the size of the \PLTLf goal (together with the problem, when needed).
    \end{enumerate}
\end{enumerate}
The \PLTLf goal formulas employed are quite common (see e.g., \citep{sohrabi2011preferred}, p. 264, col. 2): (1) one requires a set of conditions in order to occur; (2) the other requires that certain conditions were previously true in some arbitrary order (this generates a \DFA that is exponential in the number of conditions). These formulas have the advantage of being compactly translatable into the corresponding \LTLf formulas. 

For the nondeterministic part, we focus on strong-cyclic solutions since there are better tools for this kind of solution. Nevertheless, our approach applies to both strong-cyclic and strong solutions seamlessly. 
We also assume unitary cost for every domain's action $a\in A$.

\paragraph{Benchmarks}
For the deterministic part, we chose the \blocks (deterministic) and the \elevator from the 
% 2000th International Planning Competition, 
IPC-00,
whereas for the nondeterministic part we chose
\tireworld and \blocks (nondeterministic) from
% the 2006th and the 2008th International Planning Competition, 
IPC-06 and IPC-08,
respectively.\footnote{\url{https://www.icaps-conference.org/competitions/}}
For experiment \ref{exp-overhead}, we used the problems from the planning competition from which we took the domain.
For experiment \ref{exp-comparison}, we observe that there are no planning benchmarks with general \PLTLf goals,
and using the existing \LTLf goals from the literature would have been prohibitive as the best algorithm to translate from \LTLf to \PLTLf, and vice versa, is 3EXPTIME \citep{degiacomo2020pure}.
Therefore, we generated our own problems and \PLTLf goals over both deterministic and nondeterministic domains, and where the equivalent \LTLf counterpart was easy to obtain, in order to use \ltlfondtofondshort.

Moreover, it would have been interesting to make some benchmarks with the compilation tools presented in~\cite{BaierM06, torres2015polynomial} for the deterministic setting. Unfortunately, such tools are not publicly available online anymore, so we cannot fairly compare our implementation with them. Anyway, given that the compilation technique in~\cite{camacho2017nondeterministic} builds upon~\cite{BaierM06, torres2015polynomial}, we find it reasonable to assume that the performances of~\cite{BaierM06, torres2015polynomial} are analogous to the ones of~\cite{camacho2017nondeterministic}.

\paragraph{\blocks (deterministic)}
We run Experiment \ref{exp-overhead} over the 102 problems available from the planning competition. In Figure \ref{fig:blocksworld2-det},
we plot the running time of \fastdownwardshort versus \planforpastfdshort. As one can notice, the overhead introduced by our compilation technique, wrt the running time of the standard planner over the original problem, does not diverge when the problem gets harder, and the running time of the compiled planning task follows very closely the running time of the original planning task.
Regarding experiments of type \ref{exp-comparison}, we considered the number of blocks as the size of the problem, and chose a sequence goal formula parametrized with $n\ge 2$: $\varphi_n = \Once(on(b_1,b_2) \wedge \Yesterday\Once(on(b_2,b_3) \wedge \cdots \Yesterday\Once(on(b_{n-1},b_{n}))))$. Its \LTLf counterpart for \ltlfondtofondshort is simply $\varphi_n^f = \Eventually(on(b_{n-1},\\b_n) \wedge \Next\Eventually(on(b_{n-2},b_n) \wedge \cdots \Next\Eventually on(b_1, b_2))))$. The initial condition is that all the blocks are on the table and clear.
For Experiment \ref{exp-comparison-a}, we fixed the formula parameter $n=3$ and increased the number of blocks from $3$ to $20$. The results are shown in Figure \ref{fig:blocksworld1a-det}.
For Experiment \ref{exp-comparison-b}, we increased the parameter $n$ from $2$ to $20$, and the results are in Figure \ref{fig:blocksworld1b-det}.
In the former case, we note that the size of the problem does not affect the performances of our tool \planforpastfdshort and \fondforlpfdshort, except for \ltlfondtofondfdshort; and in the latter case, our tool largely outperforms its competitors, which timeout far earlier.

\begin{figure*}[!htbp]
    \centering
    \begin{subfigure}{.49\linewidth}
    \centering
    \includegraphics[width=1\textwidth]{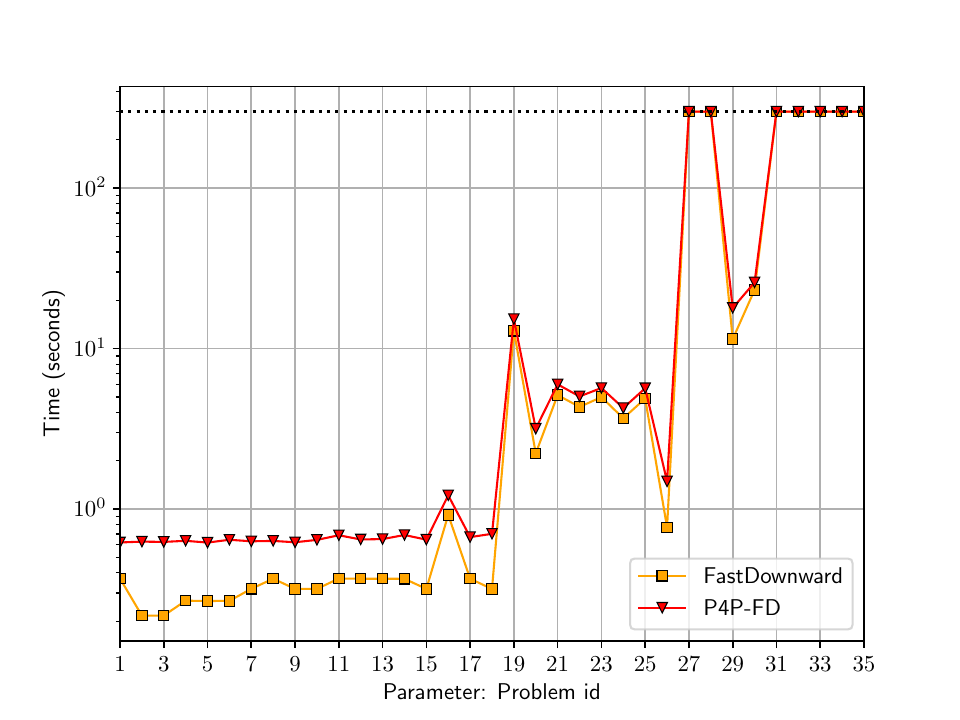}
    \caption{\scriptsize\blocks deterministic, experiment \ref{exp-overhead} (35 instances).}
    %(only first 35 instances).
    \label{fig:blocksworld2-det}
    \end{subfigure}
    \begin{subfigure}{.49\linewidth}
    \centering
    \includegraphics[width=1\textwidth]{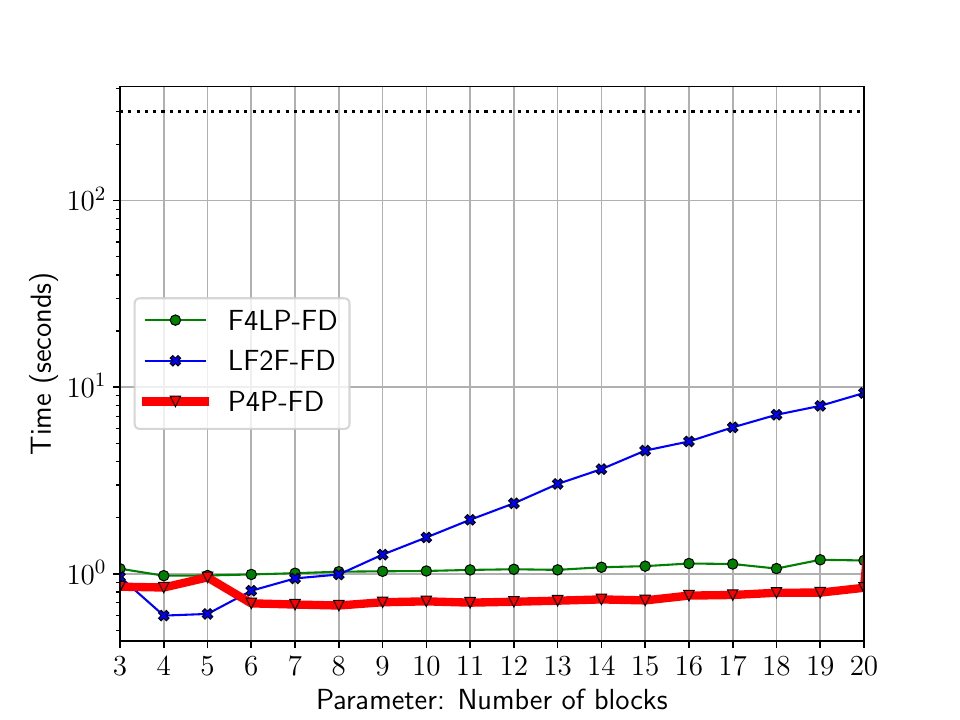}
    \caption{\scriptsize \blocks deterministic, experiment \ref{exp-comparison-a}.}
    \label{fig:blocksworld1a-det}
    \end{subfigure}
    \begin{subfigure}{.49\linewidth}
    \centering
    \includegraphics[width=1\textwidth]{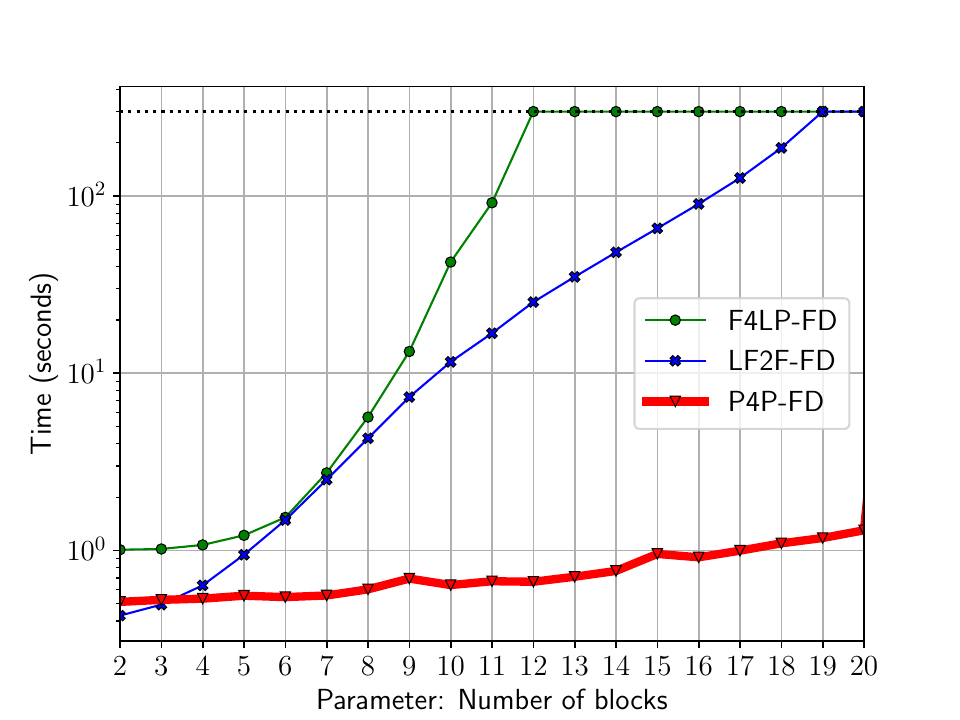}
    \caption{\scriptsize \blocks deterministic, experiment \ref{exp-comparison-b}.}
    \label{fig:blocksworld1b-det}
    \end{subfigure}
    \begin{subfigure}{.49\linewidth}
    \centering
    \includegraphics[width=1\textwidth]{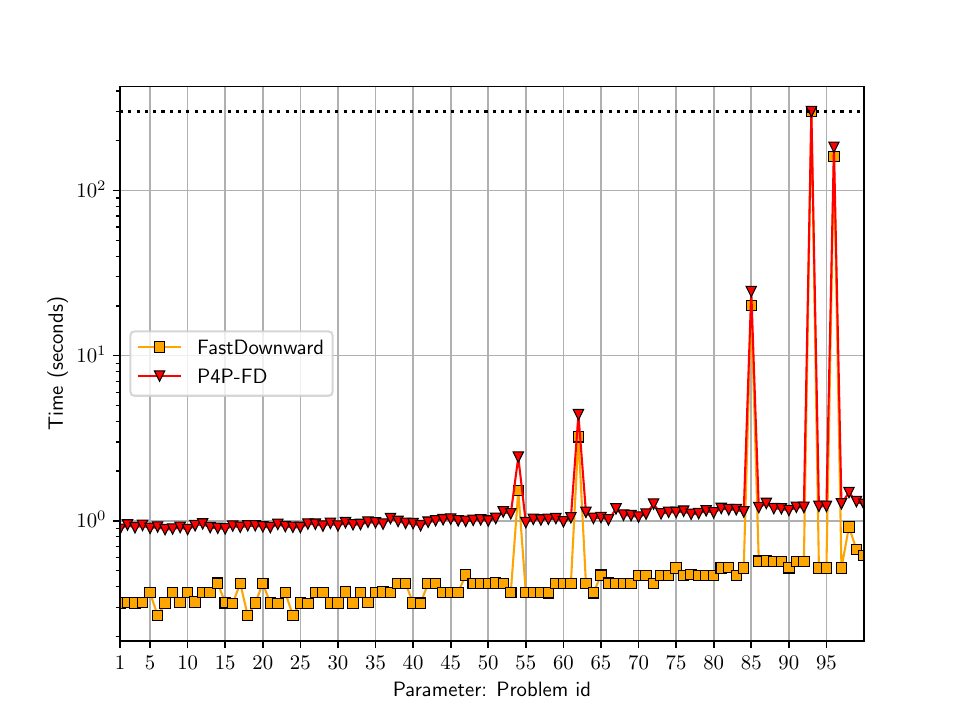}
    \caption{\scriptsize \elevator, experiment \ref{exp-overhead} (100 instances)}
    \label{fig:elevator2-det}
    \end{subfigure}
    \\
    \begin{subfigure}{.49\linewidth}
    \centering
    \includegraphics[width=1\textwidth]{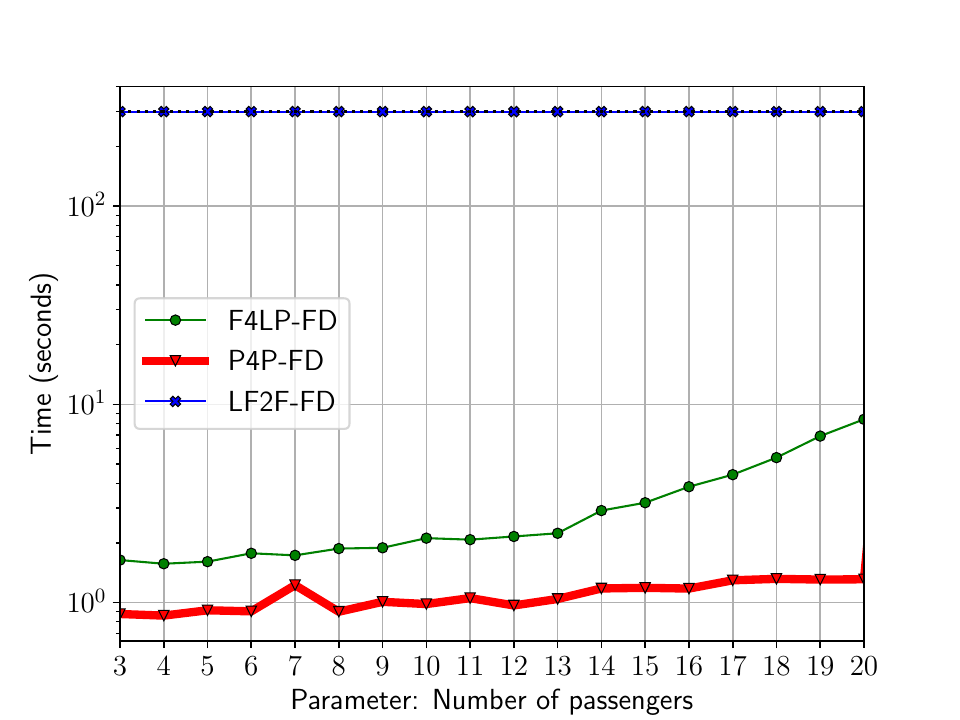}
    \caption{\scriptsize \elevator, experiment \ref{exp-comparison-a}.}
    \label{fig:elevator1a-det-20}
    \end{subfigure}
    \begin{subfigure}{.49\linewidth}
    \centering
    \includegraphics[width=1\textwidth]{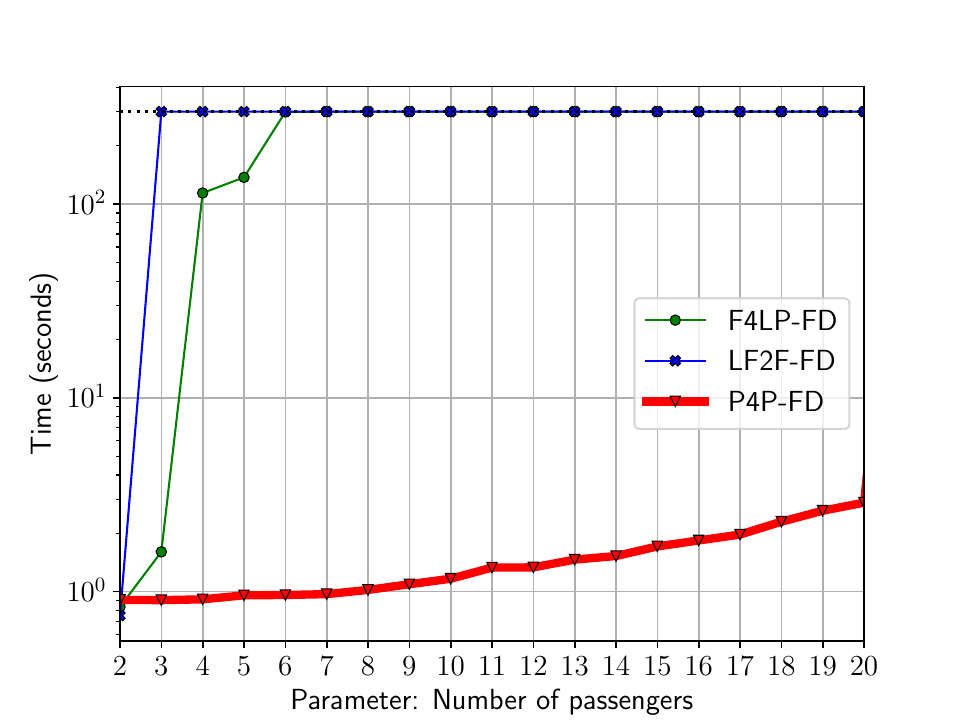}
    \caption{\scriptsize \elevator, experiment \ref{exp-comparison-b}.}
    \label{fig:elevator1b-det-20}
    \end{subfigure}
    \begin{subfigure}{.49\linewidth}
    \centering
    \includegraphics[width=1\textwidth]{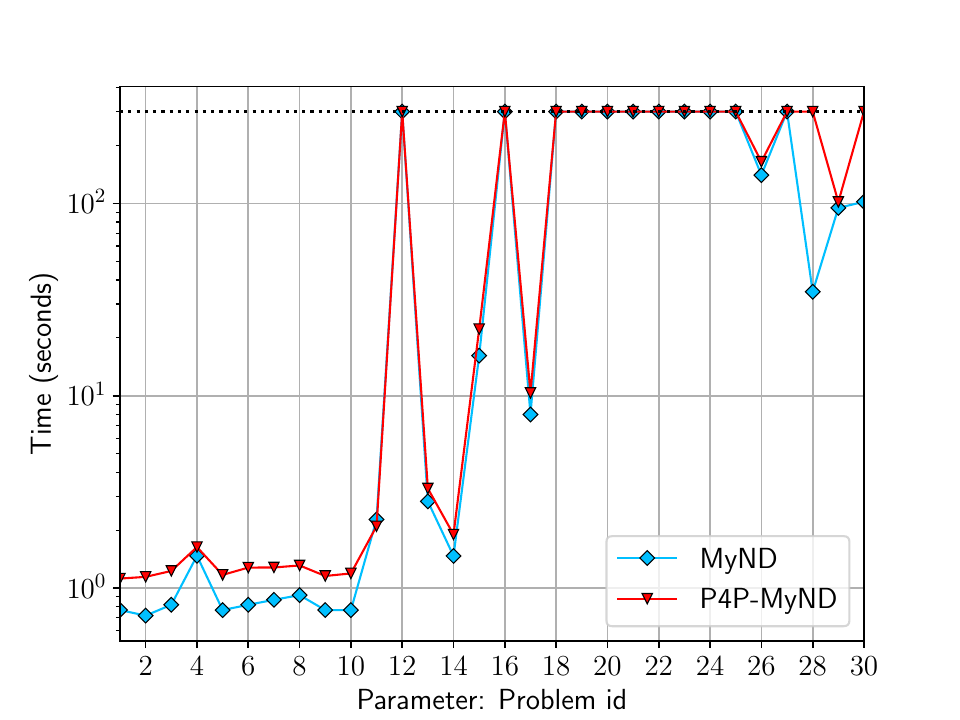}
    \caption{\scriptsize \blocksnd, experiment \ref{exp-overhead}.}
    \label{fig:blocksworld2-nondet}
    \end{subfigure}
    \begin{subfigure}{.49\linewidth}
    \centering
    \includegraphics[width=1\textwidth]{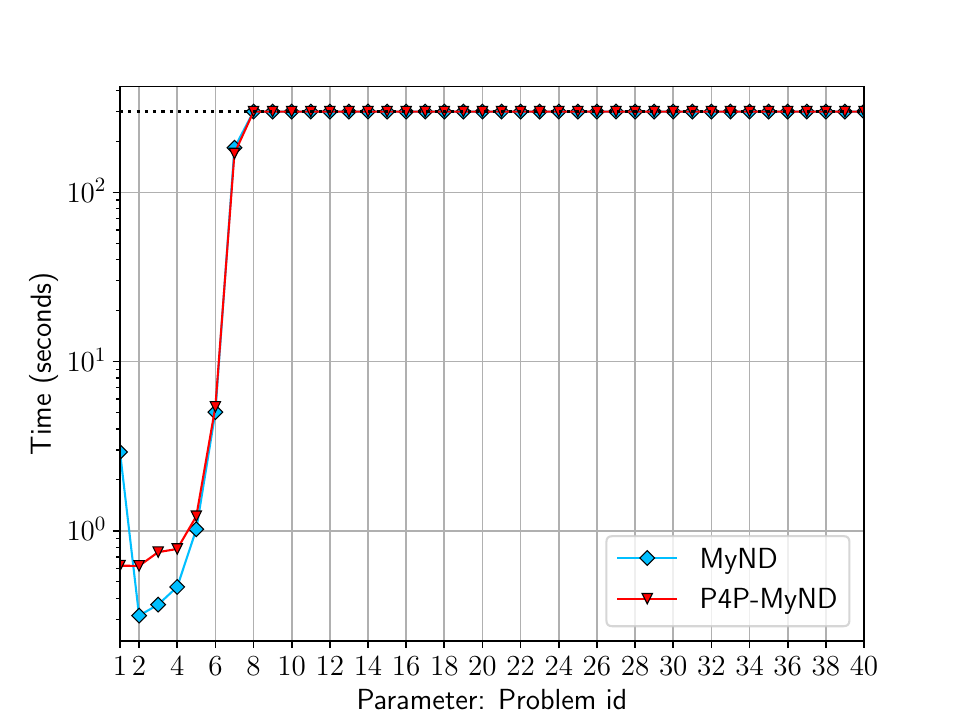}
    \caption{\scriptsize \tireworld, experiment \ref{exp-overhead}.}
    \label{fig:tireworld2-nondet}
    \end{subfigure}
    \caption{Comparison results on all benchmarks.}
\end{figure*}

% \begin{figure}
%     \centering
%     \includegraphics[width=0.49\textwidth]{images/blocksworld2-det.pdf}
%     \caption{\blocks deterministic, experiment \ref{exp-overhead} (only first 35 instances).}
%     \label{fig:blocksworld2-det}
% \end{figure}

% \begin{figure}
%     \centering
%     \includegraphics[width=0.49\textwidth]{images/blocksworld1a-det.pdf}
%     \caption{\blocks deterministic, experiment \ref{exp-comparison-a}.}
%     \label{fig:blocksworld1a-det}
% \end{figure}

% \begin{figure}
%     \centering
%     \includegraphics[width=0.49\textwidth]{images/blocksworld1b-det.pdf}
%     \caption{\blocks deterministic, experiment \ref{exp-comparison-b}.}
%     \label{fig:blocksworld1b-det}
% \end{figure}

\paragraph{\elevator}
We run Experiment \ref{exp-overhead} over the 150 problems available from the planning competition. Figure \ref{fig:elevator2-det} shows the running time of \fastdownwardshort versus \planforpastfdshort, where we observe that we get a similar result as in \blocks (deterministic).
Regarding experiments of type \ref{exp-comparison}, we considered as size of the problem the number of passengers $n$, with $2n$ floors, all passengers starting from floor $0$ with destination for a passenger $i$ the floor $2i$. The goal formula is to serve all the passengers, i.e. $\varphi_n = \Once(served(p_1)) \wedge \cdots \Once(served(p_n))$. Its \LTLf counterpart is $\varphi^f_n = \Eventually(served(p_1))\wedge\cdots \Eventually(served(p_n))$.
For Experiment \ref{exp-comparison-a}, we fixed the formula parameter $n=3$ and increased the number of passengers from $3$ to $20$. The results are shown in Figure \ref{fig:elevator1a-det-20}.
For Experiment \ref{exp-comparison-b}, we increased the parameter $n$ from $3$ to $20$, and the results are in Figure \ref{fig:elevator1b-det-20}.
In the former case, we note that the size of the problem does not affect the performances of our tool \planforpastfdshort, instead of what happens for \fondforlpfdshort and \ltlfondtofondfdshort; and in the latter case, our tool largely outperforms its competitors, which timeout far earlier.
In fact, the planner \ltlfondtofondfdshort gets stuck in the translation step, and we conjecture it is due to the high 
bookkeeping machinery introduced to produce the compiled domain and problem.
% overhead in terms of complexity of the produced compiled domain.

% \begin{figure}
    % \centering
    % \includegraphics[width=0.5\textwidth]{images/elevator2-det.pdf}
    % \caption{\elevator, experiment \ref{exp-overhead} (only first 100 instances)}
    % \label{fig:elevator2-det}
% \end{figure}

% \begin{figure}
    % \centering
    % \includegraphics[width=0.5\textwidth]{images/elevator-1a-det.pdf}
    % \caption{\elevator, experiment \ref{exp-comparison-a}.}
    % \label{fig:elevator1a-det-20}
% \end{figure}

% \begin{figure}
    % \centering
    % \includegraphics[width=0.5\textwidth]{images/elevator-1b-det.pdf}
    % \caption{\elevator, experiment \ref{exp-comparison-b}.}
    % \label{fig:elevator1b-det-20}
% \end{figure}

\paragraph{\blocks (nondeterministic)} The experimental setup is very similar to the deterministic case with the same goal description, except that the problems are taken from a different planning competition.
In Figure \ref{fig:blocksworld2-nondet}, we show the results for the experiment \ref{exp-overhead}, run over 30 problems. We note that, also in the nondeterministic case, the overhead is quite small wrt the running time of the original task.
% In Figure \ref{fig:blocksworld1a-nondet} and \ref{fig:blocksworld1b-nondet} we show, respectively, the results for experiment of type \ref{exp-comparison-a} and \ref{exp-comparison-b}.
We also run the experiments of type \ref{exp-comparison-a} and \ref{exp-comparison-b}, and we observed that our tool, backed by \mynd, scales much better than \ltlfondtofondmyndshort, and it is competitive with \fondforlpmyndshort, and sometimes better (especially in \ref{exp-comparison-a}). Due to lack of space, we do not report these results here, but they can be found in the supplementary material.

% \begin{figure}
%     \centering
%     \includegraphics[width=0.5\textwidth]{images/blocksworld1a-nondet.pdf}
%     \caption{\blocks (nondeterministic), experiment \ref{exp-comparison-a}.}
%     \label{fig:blocksworld1a-nondet}
% \end{figure}

% \begin{figure}
%     \centering
%     \includegraphics[width=0.5\textwidth]{images/blocksworld1b-nondet.pdf}
%     \caption{\blocks (nondeterministic), experiment \ref{exp-comparison-b}.}
%     \label{fig:blocksworld1b-nondet}
% \end{figure}

% \begin{figure}
%     \centering
%     \includegraphics[width=0.5\textwidth]{images/blocksworld2-nondet.pdf}
%     \caption{\blocks (nondeterministic), experiment \ref{exp-overhead}.}
%     \label{fig:blocksworld2-nondet}
% \end{figure}

\paragraph{\tireworld} 
We run experiment \ref{exp-overhead} over the 40 problem instances from the planning competition.
The results are shown in Figure \ref{fig:tireworld2-nondet}. We can see that the running time overhead is very small also in this case.
Regarding experiments of type \ref{exp-comparison}, we consider as size of the problem the length $n$ of a side of the triangle, and a spare tire in every location. The temporal goal is to visit $m$ locations in the following order: $l11, l12, l22, l23, \dots$.
E.g. the past formula for $m=3$ is: $\varphi_m = \Once(vehicle\_at(l22) \wedge \Yesterday(\Once(vehicle\_at(l21) \wedge \Yesterday(\Once(vehicle\_at(l11))))))$.
The fixed goal for experiment \ref{exp-comparison-a} is: $\Once(vehicle\_at(l31) \wedge \Yesterday(\Once(vehicle\_at(l21) \wedge \Yesterday(\Once(vehicle\_at(l12))))))$.
These domain and goals turned out to be tough for all the approaches under study, with few solved instances.
These results can be found in the supplementary material.

% \begin{figure}
%     \centering
%     \includegraphics[width=0.5\textwidth]{images/triangletireworld2.pdf}
%     \caption{\tireworld, experiment \ref{exp-comparison}.}
%     \label{fig:triangletireworld2}
% \end{figure}

% \begin{figure}
%     \centering
%     \includegraphics[width=0.5\textwidth]{images/triangletireworld1a.pdf}
%     \caption{\tireworld (nondeterministic), experiment \ref{exp-comparison-a}.}
%     \label{fig:triangletireworld1a}
% \end{figure}

% \begin{figure}
%     \centering
%     \includegraphics[width=0.5\textwidth]{images/triangletireworld1b.pdf}
%     \caption{\tireworld (nondeterministic), experiment \ref{exp-comparison-b}.}
%     \label{fig:triangletireworld1b}
% \end{figure}

\paragraph{Discussion}
These experiments show that in virtually all cases our tool performs significantly better than its competitors, in both the deterministic and nondeterministic case. We attribute this to the rather different nature of our approach and the competitors' approach: whilst \fondforlpshort and \ltlfondtofond compute the explicit \DFA of the goal formula (whose size is worst-case doubly-exponential for \LTLf and worst-case exponential for \PLTLf wrt the size of the formula) and then compile it into the new domain and problem, the compilation of our tool processes the formula directly generating only a minimal number of additional fluents, and uses them to  delegate the semantic evaluation of the \PLTLf goal to the planner. Another important point is that our technique does not introduce auxiliary actions, hence allowing the planner to work more efficiently in searching for a solution.
This is confirmed by the fact that in the experiments the number of expanded nodes is often the same of the original task, whereas for the other tools that is not the case.
\section{Discussion}
\label{sec:related-work}

% This section reviews previous work and compares them against what we presented in this paper. Among all related works, we identify some researches that use \PLTLf as a property to be achieved, optimized, or maintained in some form of planning (deterministic, nondeterministic, or MDP). These works all exploit some kind of translation of the \PLTLf formula to a more standard planning representation. Other papers describe the use of progression with \LTL in planning.

Our research shows that \PLTLf is a sweet spot in expressing temporally extended goals since it only introduces minimal overhead. Handling \PLTLf is particularly simple and elegant. A single compilation that works for deterministic and nondeterministic planning domains (and, in fact, it also works for MDPs, but this is out of the scope of the paper). Section~\ref{sec:progression} gives the mathematics behind it; Section~\ref{sec:pddl} gives an implementation directly based on the mathematics (for this elegance we need derived predicates); Section~\ref{sec:experiments} shows the practical effectiveness of the approach.  

These nice results hold for \PLTLf only. Indeed, (1) it is crucial to work with a \DFA in order not to introduce forms of decision that would impact planning (crucial in nondeterministic domains). (2) iterated progression can be thought of as an implicit form of \DFA construction (every formula has a single progression), meaning that the iterated progression must be able to create doubly-exponentially-many non-equivalent annotations to be complete in the case of \LTLf, whereas only exponentially-many (the truth-value of linearly many new fluents) for \PLTLf.
Even if we do not use automata explicitly, the connection with automata remains important in understanding how to deal with \LTLf/\PLTLf because it gives the essential information needed in order to be able to check if a trace satisfies an \LTLf/\PLTLf formula.

% R4P3P9: Avoiding \DFAs is not a contribution of the paper, progressions/regressions are implicitly building \DFAs anyway (though only the part that is requested by the forward search).

Specifically, \PLTLf translates into an exponential \DFA, while \LTLf translates into an exponential \NFA. If we consider a deterministic domain, with both \DFA and \NFA we can do the Cartesian product with the domain on-the-fly while searching for the solution (PSPACE). For \FOND, we first need to compute the \NFA and then transform it into a \DFA on-the-fly (2EXPTIME). Instead, \PLTLf remains EXPTIME. 
% Interestingly, exactly the same translation works for both deterministic and nondeterministic domains.

Although not always crisply stated, these observations are at the base of much of the related research.

\cite{bacchus1997structured} does something very similar to us. The similarity comes from the fact that both approaches are based on progression (i.e. regression) of \PLTLf, i.e. on the ``fixpoint equivalences'' based on ``now'' and ``next'' (``previously''), see e.g. the survey~\citep{Emerson1990TemporalAM}. In fact, we can use our Section~\ref{sec:progression} to formally justify the correctness of the approach in that paper (where correctness is not proved). Also, our translation could be used to implement their decision tree based transitions and rewards. This is something to investigate in future works related to non-Markovian Decision Processes. 

\cite{sohrabi2011preferred} uses \PLTLf for $\phi_{expl}$. However, (cf. p. 265, col. 2), they handle them in a na\"ive way, in view of recent understanding \citep{degiacomo2020pure}. They consider \PLTLf formulas as \LTLf formulas on the reverse trace (const), compute the \NFA for the \LTLf (exp), reverse it (poly), so the trace is in the correct direction, build the planning domain and solve the problem (PSPACE). The resulting technique is worst-case EXPSPACE, whereas it could be PSPACE. Moreover, they do not exploit the fact that one can obtain the \DFA of the reverse language in single exponential \citep{degiacomo2020pure} and on-the-fly while planning (this paper).

\cite{mallett2021progression} considers a probabilistic variant of \LTLf (without the past). Their construction is based on progression. 
% Notice that the progression is deterministic, hence to be complete, it needs to mimic a \DFA.
Since the progression is deterministic, it needs to mimic a \DFA to be complete. Given that the minimal \DFA corresponding to an \LTLf formula may contain doubly-exponential states in the worst-case, correspondingly the iterated progression needs to create doubly-exponential non-equivalent formulas to be complete. Additionally, here we cannot use \NFAs because they would interfere with probabilities. It would be interesting to use \PLTLf instead of \LTLf, since this would simplify their algorithmic part.

\cite{bienvenu2011specifying} introduces a very advanced language to talk about preferences based on \LTLf. The setting is on deterministic domains, so planning with this basic component remains PSPACE. Progression can be effectively used (as witnessed by TLPlan). However, as discussed above, iterated progression can explode in the worst-case, being deterministic, and hence does not take advantage of the fact that for this problem the \NFA would suffice, see \cite{camacho2017nondeterministic,camacho-strong-19}.

\cite{zhu2019first}, (cf. Sec. 4), studies \PLTLf and uses it to solve \LTLf through \MSO.
Note that if we use \FOL/\MSO to express temporal properties on finite traces, moving from properties as pure-past to pure-future (and vice versa) is poly (though checking \FOL/\MSO temporal properties is non-elementary). Unfortunately, if we use \PLTLf/\LTLf, although they have the same expressive power and the same expressive power of \FOL in expressing temporal properties on finite traces, moving from one to the other is in 3EXPTIME (matching lower bound is unknown).

We should note that bad complexity results on translations to \DFAs are not always mirrored as reduced performance of the systems. This is actually not often the case, e.g., the best \LTLf to \DFA translators are first based on the translation to \FOL (poly) and then use \FOL in MONA \citep{KlaEtAlMona} to obtain the \DFA (non-elementary). However, the simplicity and elegance of \PLTLf treatment stand out.

We close the section by observing that \PLTLf has a special role in the \LTL literature. For instance, the \cite{MannaPnueli89}'s temporal hierarchy is based on \PLTLf ``atoms'' $\alpha$ that are in the context of very simple future \LTL formulas: Safety, G $\alpha$; Co-Safety F $\alpha$; Liveness GF $\alpha$; Persistence FG $\alpha$; and so on. Also, an important result is the separation of temporal formulas with both past and future into a boolean combination of pure-future and pure-past formulas \citep{gabbay1994temporal}.
Finally, recent proposals of an \LTL fragment with both future and past operators are of interest. These new fragments, with atoms of arbitrary \PLTLf formulas within the scope of future operators, behave well on synthesis (related to planning) \citep{cimatti2020synthesis}.

% Finally, of interest are recent proposals of fragment of \LTL with both future and past that are well behaved wrt synthesis (related to planning), which are based on having atoms of arbitrary \PLTLf combined controlled use of future operators \citep{cimatti2020synthesis}.
\section{Conclusion}
\label{sec:conclusions}
In this paper, we have studied planning for temporally extended goals expressed in \PLTLf. \PLTLf is particularly interesting for expressing goals since it allows to express sophisticated tasks as in the Formal Methods literature, while the worst-case computational complexity of Planning in both deterministic and nondeterministic domains (\FOND) remains the same as for classical reachability goals.
Note that, this is not the case for planning in nondeterministic domains for \LTLf goals, even though \PLTLf and \LTLf share the same formal expressiveness~\citep{degiacomo2020pure}. 

We exploit this nice feature of \PLTLf to devise a direct technique for translating planning for \PLTLf goals into planning for standard reachability goals with only a minimal overhead. Experiments confirm the  effectiveness of such a technique in practice.  As a result, our technique enables state-of-the-art tools for classical, FOND strong and FOND strong-cyclic planning to handle \PLTLf goals seamlessly, essentially maintaining the performances they have for classical reachability goals.

In this paper, we have focused on \PLTLf. How to extend our results to goals given in \PLDLf, which is a strictly more expressive variant of \PLTLf \citep{degiacomo2020pure}, remains for further studies.

\section*{Acknowledgments}
This work has been partially supported by the ERC Advanced Grant WhiteMech (No. 834228), by the EU ICT-48 2020 project TAILOR (No. 952215), by the PRIN project RIPER (No. 20203FFYLK), and by the JPMorgan AI Faculty Research Award ``Resilience-based Generalized Planning and Strategic Reasoning''.

\bibliography{mybibfile}

\begin{thebibliography}{59}
\expandafter\ifx\csname natexlab\endcsname\relax\def\natexlab#1{#1}\fi
\providecommand{\url}[1]{\texttt{#1}}
\providecommand{\href}[2]{#2}
\providecommand{\path}[1]{#1}
\providecommand{\DOIprefix}{doi:}
\providecommand{\ArXivprefix}{arXiv:}
\providecommand{\URLprefix}{URL: }
\providecommand{\Pubmedprefix}{pmid:}
\providecommand{\doi}[1]{\href{http://dx.doi.org/#1}{\path{#1}}}
\providecommand{\Pubmed}[1]{\href{pmid:#1}{\path{#1}}}
\providecommand{\bibinfo}[2]{#2}
\ifx\xfnm\relax \def\xfnm[#1]{\unskip,\space#1}\fi
%Type = Article
\bibitem[{van~der Aalst et~al.(2009)van~der Aalst, Pesic \&
  Schonenberg}]{declareCSRD09}
\bibinfo{author}{van~der Aalst, W.}, \bibinfo{author}{Pesic, M.}, \&
  \bibinfo{author}{Schonenberg, H.} (\bibinfo{year}{2009}).
\newblock \bibinfo{title}{{Declarative Workflows: Balancing Between Flexibility
  and Support}}.
\newblock {\it \bibinfo{journal}{Computer Science - R\&D}\/},  {\it
  \bibinfo{volume}{23}\/}, \bibinfo{pages}{99--113}.
%Type = Article
\bibitem[{Alechina et~al.(2018)Alechina, Logan \& Dastani}]{AlechinaLD18}
\bibinfo{author}{Alechina, N.}, \bibinfo{author}{Logan, B.}, \&
  \bibinfo{author}{Dastani, M.} (\bibinfo{year}{2018}).
\newblock \bibinfo{title}{Modeling norm specification and verification in
  multiagent systems}.
\newblock {\it \bibinfo{journal}{{FLAP}}\/},  {\it \bibinfo{volume}{5}\/}.
%Type = Inproceedings
\bibitem[{Aminof et~al.(2020)Aminof, {De Giacomo} \& Rubin}]{AminofGR20}
\bibinfo{author}{Aminof, B.}, \bibinfo{author}{{De Giacomo}, G.}, \&
  \bibinfo{author}{Rubin, S.} (\bibinfo{year}{2020}).
\newblock \bibinfo{title}{Stochastic fairness and language-theoretic fairness
  in planning in nondeterministic domains}.
\newblock In {\it \bibinfo{booktitle}{{ICAPS}}\/} (pp.
  \bibinfo{pages}{20--28}).
\newblock \bibinfo{publisher}{{AAAI} Press}.
%Type = Inproceedings
\bibitem[{Bacchus et~al.(1996)Bacchus, Boutilier \& Grove}]{Bacchus96}
\bibinfo{author}{Bacchus, F.}, \bibinfo{author}{Boutilier, C.}, \&
  \bibinfo{author}{Grove, A.} (\bibinfo{year}{1996}).
\newblock \bibinfo{title}{Rewarding behaviors}.
\newblock In {\it \bibinfo{booktitle}{AAAI}\/} (pp.
  \bibinfo{pages}{1160--1167}).
%Type = Inproceedings
\bibitem[{Bacchus et~al.(1997)Bacchus, Boutilier \&
  Grove}]{bacchus1997structured}
\bibinfo{author}{Bacchus, F.}, \bibinfo{author}{Boutilier, C.}, \&
  \bibinfo{author}{Grove, A.} (\bibinfo{year}{1997}).
\newblock \bibinfo{title}{Structured solution methods for non-markovian
  decision processes}.
\newblock In {\it \bibinfo{booktitle}{{AAAI}}\/} (pp.
  \bibinfo{pages}{112--117}).
\newblock \bibinfo{organization}{Citeseer}.
%Type = Article
\bibitem[{Bacchus \& Kabanza(1998)}]{BacchusK98}
\bibinfo{author}{Bacchus, F.}, \& \bibinfo{author}{Kabanza, F.}
  (\bibinfo{year}{1998}).
\newblock \bibinfo{title}{Planning for temporally extended goals}.
\newblock {\it \bibinfo{journal}{Ann. Math. Artif. Intell.}\/},  {\it
  \bibinfo{volume}{22}\/}, \bibinfo{pages}{5--27}.
%Type = Article
\bibitem[{Bacchus \& Kabanza(2000)}]{BacchusK00}
\bibinfo{author}{Bacchus, F.}, \& \bibinfo{author}{Kabanza, F.}
  (\bibinfo{year}{2000}).
\newblock \bibinfo{title}{Using temporal logics to express search control
  knowledge for planning}.
\newblock {\it \bibinfo{journal}{Artif. Intell.}\/},  {\it
  \bibinfo{volume}{116}\/}, \bibinfo{pages}{123--191}.
%Type = Book
\bibitem[{Baier et~al.(2008{\natexlab{a}})Baier, Katoen \&
  Guldstrand~Larsen}]{BaKG08}
\bibinfo{author}{Baier, C.}, \bibinfo{author}{Katoen, J.-P.}, \&
  \bibinfo{author}{Guldstrand~Larsen, K.} (\bibinfo{year}{2008}{\natexlab{a}}).
\newblock {\it \bibinfo{title}{Principles of Model Checking}\/}.
%Type = Inproceedings
\bibitem[{Baier et~al.(2008{\natexlab{b}})Baier, Fritz, Bienvenu \&
  McIlraith}]{BaierFBM08}
\bibinfo{author}{Baier, J.~A.}, \bibinfo{author}{Fritz, C.},
  \bibinfo{author}{Bienvenu, M.}, \& \bibinfo{author}{McIlraith, S.~A.}
  (\bibinfo{year}{2008}{\natexlab{b}}).
\newblock \bibinfo{title}{Beyond classical planning: Procedural control
  knowledge and preferences in state-of-the-art planners}.
\newblock In {\it \bibinfo{booktitle}{{AAAI}}\/} (pp.
  \bibinfo{pages}{1509--1512}).
\newblock \bibinfo{publisher}{{AAAI}}.
%Type = Inproceedings
\bibitem[{Baier \& McIlraith(2006{\natexlab{a}})}]{baier2006planning}
\bibinfo{author}{Baier, J.~A.}, \& \bibinfo{author}{McIlraith, S.~A.}
  (\bibinfo{year}{2006}{\natexlab{a}}).
\newblock \bibinfo{title}{Planning with first-order temporally extended goals
  using heuristic search}.
\newblock In {\it \bibinfo{booktitle}{{AAAI}}\/} (pp.
  \bibinfo{pages}{788--795}).
\newblock \bibinfo{publisher}{{AAAI}}.
%Type = Inproceedings
\bibitem[{Baier \& McIlraith(2006{\natexlab{b}})}]{BaierM06}
\bibinfo{author}{Baier, J.~A.}, \& \bibinfo{author}{McIlraith, S.~A.}
  (\bibinfo{year}{2006}{\natexlab{b}}).
\newblock \bibinfo{title}{Planning with temporally extended goals using
  heuristic search}.
\newblock In {\it \bibinfo{booktitle}{{ICAPS}}\/} (pp.
  \bibinfo{pages}{342--345}).
\newblock \bibinfo{publisher}{{AAAI}}.
%Type = Article
\bibitem[{Bienvenu et~al.(2011)Bienvenu, Fritz \&
  McIlraith}]{bienvenu2011specifying}
\bibinfo{author}{Bienvenu, M.}, \bibinfo{author}{Fritz, C.}, \&
  \bibinfo{author}{McIlraith, S.~A.} (\bibinfo{year}{2011}).
\newblock \bibinfo{title}{Specifying and computing preferred plans}.
\newblock {\it \bibinfo{journal}{Artificial Intelligence}\/},  {\it
  \bibinfo{volume}{175}\/}, \bibinfo{pages}{1308--1345}.
%Type = Inproceedings
\bibitem[{Borgwardt et~al.(2022)Borgwardt, Hoffmann, Kovtunova, Kr{\"o}tzsch,
  Nebel \& Steinmetz}]{borgwardt2022expressivity}
\bibinfo{author}{Borgwardt, S.}, \bibinfo{author}{Hoffmann, J.},
  \bibinfo{author}{Kovtunova, A.}, \bibinfo{author}{Kr{\"o}tzsch, M.},
  \bibinfo{author}{Nebel, B.}, \& \bibinfo{author}{Steinmetz, M.}
  (\bibinfo{year}{2022}).
\newblock \bibinfo{title}{Expressivity of planning with horn description logic
  ontologies (technical report)}.
\newblock In {\it \bibinfo{booktitle}{{AAAI}}\/}.
%Type = Inproceedings
\bibitem[{Brafman \& {De Giacomo}(2019{\natexlab{a}})}]{BrafmanG19planning}
\bibinfo{author}{Brafman, R.~I.}, \& \bibinfo{author}{{De Giacomo}, G.}
  (\bibinfo{year}{2019}{\natexlab{a}}).
\newblock \bibinfo{title}{Planning for {LTLf/LDLf} goals in non-markovian fully
  observable nondeterministic domains}.
\newblock In {\it \bibinfo{booktitle}{{IJCAI}}\/} (pp.
  \bibinfo{pages}{1602--1608}).
\newblock \bibinfo{publisher}{ijcai.org}.
%Type = Inproceedings
\bibitem[{Brafman \& {De Giacomo}(2019{\natexlab{b}})}]{BrafmanG19regular}
\bibinfo{author}{Brafman, R.~I.}, \& \bibinfo{author}{{De Giacomo}, G.}
  (\bibinfo{year}{2019}{\natexlab{b}}).
\newblock \bibinfo{title}{Regular decision processes: {A} model for
  non-markovian domains}.
\newblock In {\it \bibinfo{booktitle}{{IJCAI}}\/} (pp.
  \bibinfo{pages}{5516--5522}).
\newblock \bibinfo{publisher}{ijcai.org}.
%Type = Inproceedings
\bibitem[{Brafman et~al.(2018)Brafman, {De Giacomo} \&
  Patrizi}]{BrafmanGP18rewards}
\bibinfo{author}{Brafman, R.~I.}, \bibinfo{author}{{De Giacomo}, G.}, \&
  \bibinfo{author}{Patrizi, F.} (\bibinfo{year}{2018}).
\newblock \bibinfo{title}{{LTLf/LDLf} non-markovian rewards}.
\newblock In {\it \bibinfo{booktitle}{{AAAI}}\/} (pp.
  \bibinfo{pages}{1771--1778}).
\newblock \bibinfo{publisher}{{AAAI} Press}.
%Type = Article
\bibitem[{Bryce \& Buffet(2008)}]{bryce2008international}
\bibinfo{author}{Bryce, D.}, \& \bibinfo{author}{Buffet, O.}
  (\bibinfo{year}{2008}).
\newblock \bibinfo{title}{6th international planning competition: Uncertainty
  part}.
\newblock {\it \bibinfo{journal}{IPC'08}\/}, .
%Type = Article
\bibitem[{Bylander(1994)}]{Bylander94}
\bibinfo{author}{Bylander, T.} (\bibinfo{year}{1994}).
\newblock \bibinfo{title}{The computational complexity of propositional strips
  planning}.
\newblock {\it \bibinfo{journal}{Artif. Intell.}\/},  {\it
  \bibinfo{volume}{69}\/}, \bibinfo{pages}{165--204}.
%Type = Inproceedings
\bibitem[{Camacho et~al.(2018)Camacho, Baier, Muise \&
  McIlraith}]{camacho2018finite}
\bibinfo{author}{Camacho, A.}, \bibinfo{author}{Baier, J.~A.},
  \bibinfo{author}{Muise, C.~J.}, \& \bibinfo{author}{McIlraith, S.~A.}
  (\bibinfo{year}{2018}).
\newblock \bibinfo{title}{Finite {LTL} synthesis as planning}.
\newblock In {\it \bibinfo{booktitle}{{ICAPS}}\/} (pp.
  \bibinfo{pages}{29--38}).
\newblock \bibinfo{publisher}{{AAAI} Press}.
%Type = Inproceedings
\bibitem[{Camacho et~al.(2019{\natexlab{a}})Camacho, Bienvenu \&
  McIlraith}]{camacho2019towards}
\bibinfo{author}{Camacho, A.}, \bibinfo{author}{Bienvenu, M.}, \&
  \bibinfo{author}{McIlraith, S.~A.} (\bibinfo{year}{2019}{\natexlab{a}}).
\newblock \bibinfo{title}{Towards a unified view of ai planning and reactive
  synthesis}.
\newblock In {\it \bibinfo{booktitle}{{ICAPS}}\/} (pp.
  \bibinfo{pages}{58--67}).
\newblock volume~\bibinfo{volume}{29}.
%Type = Inproceedings
\bibitem[{Camacho \& McIlraith(2019)}]{camacho-strong-19}
\bibinfo{author}{Camacho, A.}, \& \bibinfo{author}{McIlraith, S.~A.}
  (\bibinfo{year}{2019}).
\newblock \bibinfo{title}{Strong fully observable non-deterministic planning
  with ltl and ltl-f goals}.
\newblock In {\it \bibinfo{booktitle}{IJCAI}\/} (pp.
  \bibinfo{pages}{5523--5531}).
%Type = Inproceedings
\bibitem[{Camacho et~al.(2019{\natexlab{b}})Camacho, {Toro Icarte}, Klassen,
  Valenzano \& McIlraith}]{CamachoIKVM19}
\bibinfo{author}{Camacho, A.}, \bibinfo{author}{{Toro Icarte}, R.},
  \bibinfo{author}{Klassen, T.~Q.}, \bibinfo{author}{Valenzano, R.~A.}, \&
  \bibinfo{author}{McIlraith, S.~A.} (\bibinfo{year}{2019}{\natexlab{b}}).
\newblock \bibinfo{title}{{LTL} and beyond: Formal languages for reward
  function specification in reinforcement learning}.
\newblock In {\it \bibinfo{booktitle}{{IJCAI}}\/} (pp.
  \bibinfo{pages}{6065--6073}).
\newblock \bibinfo{publisher}{ijcai.org}.
%Type = Inproceedings
\bibitem[{Camacho et~al.(2017)Camacho, Triantafillou, Muise, Baier \&
  McIlraith}]{camacho2017nondeterministic}
\bibinfo{author}{Camacho, A.}, \bibinfo{author}{Triantafillou, E.},
  \bibinfo{author}{Muise, C.~J.}, \bibinfo{author}{Baier, J.~A.}, \&
  \bibinfo{author}{McIlraith, S.~A.} (\bibinfo{year}{2017}).
\newblock \bibinfo{title}{Non-deterministic planning with temporally extended
  goals: {LTL} over finite and infinite traces}.
\newblock In {\it \bibinfo{booktitle}{{AAAI}}\/} (pp.
  \bibinfo{pages}{3716--3724}).
\newblock \bibinfo{publisher}{{AAAI} Press}.
%Type = Article
\bibitem[{Chandra et~al.(1981)Chandra, Kozen \& Stockmeyer}]{CKS81}
\bibinfo{author}{Chandra, A.}, \bibinfo{author}{Kozen, D.}, \&
  \bibinfo{author}{Stockmeyer, L.} (\bibinfo{year}{1981}).
\newblock \bibinfo{title}{Alternation}.
\newblock {\it \bibinfo{journal}{J.\ of the {ACM}}\/},  {\it
  \bibinfo{volume}{28}\/}.
%Type = Inproceedings
\bibitem[{Cimatti et~al.(2020)Cimatti, Geatti, Gigante, Montanari \&
  Tonetta}]{cimatti2020synthesis}
\bibinfo{author}{Cimatti, A.}, \bibinfo{author}{Geatti, L.},
  \bibinfo{author}{Gigante, N.}, \bibinfo{author}{Montanari, A.}, \&
  \bibinfo{author}{Tonetta, S.} (\bibinfo{year}{2020}).
\newblock \bibinfo{title}{Reactive synthesis from extended bounded response
  {LTL} specifications}.
\newblock In {\it \bibinfo{booktitle}{{FMCAD}}\/} (pp.
  \bibinfo{pages}{83--92}).
\newblock \bibinfo{publisher}{{IEEE}}.
%Type = Inproceedings
\bibitem[{Cimatti et~al.(1997)Cimatti, Giunchiglia, Giunchiglia \&
  Traverso}]{CimattiGGT97}
\bibinfo{author}{Cimatti, A.}, \bibinfo{author}{Giunchiglia, F.},
  \bibinfo{author}{Giunchiglia, E.}, \& \bibinfo{author}{Traverso, P.}
  (\bibinfo{year}{1997}).
\newblock \bibinfo{title}{Planning via model checking: {A} decision procedure
  for \emph{AR}}.
\newblock In {\it \bibinfo{booktitle}{{ECP}}\/} (pp.
  \bibinfo{pages}{130--142}).
\newblock \bibinfo{publisher}{Springer} volume \bibinfo{volume}{1348} of {\it
  \bibinfo{series}{Lect. Notes Comput. Sci.}\/}.
%Type = Article
\bibitem[{Cimatti et~al.(2003)Cimatti, Pistore, Roveri \&
  Traverso}]{CimattiPRT03}
\bibinfo{author}{Cimatti, A.}, \bibinfo{author}{Pistore, M.},
  \bibinfo{author}{Roveri, M.}, \& \bibinfo{author}{Traverso, P.}
  (\bibinfo{year}{2003}).
\newblock \bibinfo{title}{Weak, strong, and strong cyclic planning via symbolic
  model checking}.
\newblock {\it \bibinfo{journal}{Artif. Intell.}\/},  {\it
  \bibinfo{volume}{147}\/}, \bibinfo{pages}{35--84}.
%Type = Inproceedings
\bibitem[{Cimatti et~al.(1998)Cimatti, Roveri \& Traverso}]{CimattiRT98}
\bibinfo{author}{Cimatti, A.}, \bibinfo{author}{Roveri, M.}, \&
  \bibinfo{author}{Traverso, P.} (\bibinfo{year}{1998}).
\newblock \bibinfo{title}{Strong planning in non-deterministic domains via
  model checking}.
\newblock In {\it \bibinfo{booktitle}{{AIPS}}\/} (pp. \bibinfo{pages}{36--43}).
\newblock \bibinfo{publisher}{{AAAI}}.
%Type = Inproceedings
\bibitem[{{De Giacomo} et~al.(2014){De Giacomo}, {De Masellis} \&
  Montali}]{de2014insensitivity}
\bibinfo{author}{{De Giacomo}, G.}, \bibinfo{author}{{De Masellis}, R.}, \&
  \bibinfo{author}{Montali, M.} (\bibinfo{year}{2014}).
\newblock \bibinfo{title}{Reasoning on ltl on finite traces: Insensitivity to
  infiniteness}.
\newblock In {\it \bibinfo{booktitle}{{AAAI}}\/}.
\newblock volume~\bibinfo{volume}{28}.
%Type = Inproceedings
\bibitem[{{De Giacomo} et~al.(2020){De Giacomo}, {Di Stasio}, Fuggitti \&
  Rubin}]{degiacomo2020pure}
\bibinfo{author}{{De Giacomo}, G.}, \bibinfo{author}{{Di Stasio}, A.},
  \bibinfo{author}{Fuggitti, F.}, \& \bibinfo{author}{Rubin, S.}
  (\bibinfo{year}{2020}).
\newblock \bibinfo{title}{Pure-past linear temporal and dynamic logic on finite
  traces}.
\newblock In {\it \bibinfo{booktitle}{{IJCAI}}\/}.
\newblock volume~\bibinfo{volume}{20}.
%Type = Inproceedings
\bibitem[{{De Giacomo} \& Fuggitti(2021)}]{fond4ltlf}
\bibinfo{author}{{De Giacomo}, G.}, \& \bibinfo{author}{Fuggitti, F.}
  (\bibinfo{year}{2021}).
\newblock \bibinfo{title}{Fond4ltlf: Fond planning for ltlf/pltlf goals as a
  service}.
\newblock In {\it \bibinfo{booktitle}{ICAPS, demo track}\/}.
%Type = Inproceedings
\bibitem[{{De Giacomo} et~al.(2019){De Giacomo}, Iocchi, Favorito \&
  Patrizi}]{DeGiacomoIFP19}
\bibinfo{author}{{De Giacomo}, G.}, \bibinfo{author}{Iocchi, L.},
  \bibinfo{author}{Favorito, M.}, \& \bibinfo{author}{Patrizi, F.}
  (\bibinfo{year}{2019}).
\newblock \bibinfo{title}{Foundations for restraining bolts: Reinforcement
  learning with {LTLf/LDLf} restraining specifications}.
\newblock In {\it \bibinfo{booktitle}{{ICAPS}}\/} (pp.
  \bibinfo{pages}{128--136}).
\newblock \bibinfo{publisher}{{AAAI} Press}.
%Type = Inproceedings
\bibitem[{{De Giacomo} \& Rubin(2018)}]{degiacomo2018automata}
\bibinfo{author}{{De Giacomo}, G.}, \& \bibinfo{author}{Rubin, S.}
  (\bibinfo{year}{2018}).
\newblock \bibinfo{title}{Automata-theoretic foundations of {FOND} planning for
  ltlf and ldlf goals}.
\newblock In {\it \bibinfo{booktitle}{{IJCAI}}\/} (pp.
  \bibinfo{pages}{4729--4735}).
%Type = Inproceedings
\bibitem[{{De Giacomo} \& Vardi(2013)}]{DegVa13}
\bibinfo{author}{{De Giacomo}, G.}, \& \bibinfo{author}{Vardi, M.}
  (\bibinfo{year}{2013}).
\newblock \bibinfo{title}{Linear temporal logic and linear dynamic logic on
  finite traces}.
\newblock In {\it \bibinfo{booktitle}{IJCAI}\/}.
%Type = Inproceedings
\bibitem[{{De Giacomo} \& Vardi(2015)}]{DegVa15}
\bibinfo{author}{{De Giacomo}, G.}, \& \bibinfo{author}{Vardi, M.}
  (\bibinfo{year}{2015}).
\newblock \bibinfo{title}{Synthesis for {LTL} and {LDL} on finite traces}.
\newblock In {\it \bibinfo{booktitle}{IJCAI}\/}.
%Type = Inproceedings
\bibitem[{{De Giacomo} \& Vardi(1999)}]{DeGiacomoV99}
\bibinfo{author}{{De Giacomo}, G.}, \& \bibinfo{author}{Vardi, M.~Y.}
  (\bibinfo{year}{1999}).
\newblock \bibinfo{title}{Automata-theoretic approach to planning for
  temporally extended goals}.
\newblock In {\it \bibinfo{booktitle}{{ECP}}\/} (pp.
  \bibinfo{pages}{226--238}).
\newblock \bibinfo{publisher}{Springer} volume \bibinfo{volume}{1809} of {\it
  \bibinfo{series}{Lect. Notes Comput. Sci.}\/}.
%Type = Inproceedings
\bibitem[{Emerson(1990)}]{Emerson1990TemporalAM}
\bibinfo{author}{Emerson, E.~A.} (\bibinfo{year}{1990}).
\newblock \bibinfo{title}{Temporal and modal logic}.
\newblock In {\it \bibinfo{booktitle}{Handbook of Theoretical Computer Science,
  Volume B: Formal Models and Sematics}\/}.
%Type = Incollection
\bibitem[{Fisher \& Wooldridge(2005)}]{FisherW05}
\bibinfo{author}{Fisher, M.}, \& \bibinfo{author}{Wooldridge, M.}
  (\bibinfo{year}{2005}).
\newblock \bibinfo{title}{Temporal reasoning in agent-based systems}.
\newblock In {\it \bibinfo{booktitle}{H. of Temporal Reasoning in AI}\/}.
%Type = Article
\bibitem[{Gabaldon(2011)}]{Gabaldon11}
\bibinfo{author}{Gabaldon, A.} (\bibinfo{year}{2011}).
\newblock \bibinfo{title}{Non-{Markovian} control in the situation calculus}.
\newblock {\it \bibinfo{journal}{Artif. Intell.}\/},  {\it
  \bibinfo{volume}{175}\/}, \bibinfo{pages}{25--48}.
%Type = Misc
\bibitem[{Gabbay et~al.(1994)Gabbay, Hodkinson \&
  Reynolds}]{gabbay1994temporal}
\bibinfo{author}{Gabbay, D.~M.}, \bibinfo{author}{Hodkinson, I.}, \&
  \bibinfo{author}{Reynolds, M.} (\bibinfo{year}{1994}).
\newblock \bibinfo{title}{Temporal logic: mathematical foundations and
  computational aspects}.
%Type = Book
\bibitem[{Geffner \& Bonet(2013)}]{geffner2013concise}
\bibinfo{author}{Geffner, H.}, \& \bibinfo{author}{Bonet, B.}
  (\bibinfo{year}{2013}).
\newblock {\it \bibinfo{title}{A Concise Introduction to Models and Methods for
  Automated Planning}\/}.
%Type = Article
\bibitem[{Gerevini et~al.(2009)Gerevini, Haslum, Long, Saetti \&
  Dimopoulos}]{GereviniHLSD09}
\bibinfo{author}{Gerevini, A.}, \bibinfo{author}{Haslum, P.},
  \bibinfo{author}{Long, D.}, \bibinfo{author}{Saetti, A.}, \&
  \bibinfo{author}{Dimopoulos, Y.} (\bibinfo{year}{2009}).
\newblock \bibinfo{title}{Deterministic planning in the fifth international
  planning competition: {PDDL3} and experimental evaluation of the planners}.
\newblock {\it \bibinfo{journal}{Artif. Intell.}\/},  {\it
  \bibinfo{volume}{173}\/}, \bibinfo{pages}{619--668}.
%Type = Inproceedings
\bibitem[{Giunchiglia \& Traverso(1999)}]{GiunchigliaT99}
\bibinfo{author}{Giunchiglia, F.}, \& \bibinfo{author}{Traverso, P.}
  (\bibinfo{year}{1999}).
\newblock \bibinfo{title}{Planning as model checking}.
\newblock In {\it \bibinfo{booktitle}{{ECP}}\/}.
\newblock volume \bibinfo{volume}{1809} of {\it \bibinfo{series}{Lect. Notes
  Comput. Sci.}\/}.
%Type = Article
\bibitem[{Helmert(2006)}]{helmertFD}
\bibinfo{author}{Helmert, M.} (\bibinfo{year}{2006}).
\newblock \bibinfo{title}{The fast downward planning system}.
\newblock {\it \bibinfo{journal}{J. Artif. Int. Res.}\/},  {\it
  \bibinfo{volume}{26}\/}, \bibinfo{pages}{191–246}.
%Type = Inproceedings
\bibitem[{Henriksen et~al.(1995)Henriksen, Jensen, J{\o}rgensen, Klarlund,
  Paige, Rauhe \& Sandholm}]{KlaEtAlMona}
\bibinfo{author}{Henriksen, J.}, \bibinfo{author}{Jensen, J.},
  \bibinfo{author}{J{\o}rgensen, M.}, \bibinfo{author}{Klarlund, N.},
  \bibinfo{author}{Paige, B.}, \bibinfo{author}{Rauhe, T.}, \&
  \bibinfo{author}{Sandholm, A.} (\bibinfo{year}{1995}).
\newblock \bibinfo{title}{Mona: Monadic second-order logic in practice}.
\newblock In {\it \bibinfo{booktitle}{Tools and Algorithms for the Construction
  and Analysis of Systems, First International Workshop, TACAS '95, LNCS
  1019}\/}.
%Type = Article
\bibitem[{Hoffmann \& Edelkamp(2005)}]{hoffmann2005deterministic}
\bibinfo{author}{Hoffmann, J.}, \& \bibinfo{author}{Edelkamp, S.}
  (\bibinfo{year}{2005}).
\newblock \bibinfo{title}{The deterministic part of {IPC-4}: An overview}.
\newblock {\it \bibinfo{journal}{J. Artif. Int. Res.}\/},  {\it
  \bibinfo{volume}{24}\/}, \bibinfo{pages}{519--579}.
%Type = Inproceedings
\bibitem[{Knobbout et~al.(2016)Knobbout, Dastani \&
  Meyer}]{knobbout2016dynamic}
\bibinfo{author}{Knobbout, M.}, \bibinfo{author}{Dastani, M.}, \&
  \bibinfo{author}{Meyer, J.} (\bibinfo{year}{2016}).
\newblock \bibinfo{title}{A dynamic logic of norm change}.
\newblock In {\it \bibinfo{booktitle}{ECAI}\/}.
%Type = Inproceedings
\bibitem[{Lichtenstein et~al.(1985)Lichtenstein, Pnueli \&
  Zuck}]{LichtensteinPZ85}
\bibinfo{author}{Lichtenstein, O.}, \bibinfo{author}{Pnueli, A.}, \&
  \bibinfo{author}{Zuck, L.~D.} (\bibinfo{year}{1985}).
\newblock \bibinfo{title}{The glory of the past}.
\newblock In {\it \bibinfo{booktitle}{Logic of Programs}\/} (pp.
  \bibinfo{pages}{196--218}).
\newblock \bibinfo{publisher}{Springer} volume \bibinfo{volume}{193} of {\it
  \bibinfo{series}{Lect. Notes Comput. Sci.}\/}.
%Type = Inproceedings
\bibitem[{Mallett et~al.(2021)Mallett, Thi{\'e}baux \&
  Trevizan}]{mallett2021progression}
\bibinfo{author}{Mallett, I.}, \bibinfo{author}{Thi{\'e}baux, S.}, \&
  \bibinfo{author}{Trevizan, F.} (\bibinfo{year}{2021}).
\newblock \bibinfo{title}{Progression heuristics for planning with
  probabilistic ltl constraints}.
\newblock In {\it \bibinfo{booktitle}{{AAAI}}\/} (pp.
  \bibinfo{pages}{11870--11879}).
\newblock volume~\bibinfo{volume}{35}.
%Type = Inproceedings
\bibitem[{Manna \& Pnueli(1990)}]{MannaPnueli89}
\bibinfo{author}{Manna, Z.}, \& \bibinfo{author}{Pnueli, A.}
  (\bibinfo{year}{1990}).
\newblock \bibinfo{title}{A hierarchy of temporal properties}.
\newblock In {\it \bibinfo{booktitle}{{PODC}}\/} (pp.
  \bibinfo{pages}{377--410}).
\newblock \bibinfo{publisher}{{ACM}}.
%Type = Article
\bibitem[{Mattm{\"u}ller et~al.(2010)Mattm{\"u}ller, Ortlieb, Helmert \&
  Bercher}]{mattmuller2010pattern}
\bibinfo{author}{Mattm{\"u}ller, R.}, \bibinfo{author}{Ortlieb, M.},
  \bibinfo{author}{Helmert, M.}, \& \bibinfo{author}{Bercher, P.}
  (\bibinfo{year}{2010}).
\newblock \bibinfo{title}{Pattern database heuristics for fully observable
  nondeterministic planning}.
\newblock {\it \bibinfo{journal}{ICAPS}\/},  {\it \bibinfo{volume}{20}\/},
  \bibinfo{pages}{105--112}.
%Type = Techreport
\bibitem[{McDermott et~al.(1998)McDermott, Ghallab, Howe, Knoblock, Ram,
  Veloso, Weld \& Wilkins}]{mcdermott1998pddl}
\bibinfo{author}{McDermott, D.}, \bibinfo{author}{Ghallab, M.},
  \bibinfo{author}{Howe, A.}, \bibinfo{author}{Knoblock, C.},
  \bibinfo{author}{Ram, A.}, \bibinfo{author}{Veloso, M.},
  \bibinfo{author}{Weld, D.}, \& \bibinfo{author}{Wilkins, D.}
  (\bibinfo{year}{1998}).
\newblock {\it \bibinfo{title}{PDDL -- the planning domain definition
  language}\/}.
\newblock \bibinfo{type}{Technical Report} {ICAPS}.
%Type = Inproceedings
\bibitem[{Muise et~al.(2012)Muise, McIlraith \& Beck}]{muise2012improved}
\bibinfo{author}{Muise, C.}, \bibinfo{author}{McIlraith, S.}, \&
  \bibinfo{author}{Beck, C.} (\bibinfo{year}{2012}).
\newblock \bibinfo{title}{Improved non-deterministic planning by exploiting
  state relevance}.
\newblock In {\it \bibinfo{booktitle}{ICAPS}\/}.
\newblock volume~\bibinfo{volume}{22}.
%Type = Inproceedings
\bibitem[{Pistore \& Traverso(2001)}]{PistoreT01}
\bibinfo{author}{Pistore, M.}, \& \bibinfo{author}{Traverso, P.}
  (\bibinfo{year}{2001}).
\newblock \bibinfo{title}{Planning as model checking for extended goals in
  non-deterministic domains}.
\newblock In {\it \bibinfo{booktitle}{{IJCAI}}\/} (pp.
  \bibinfo{pages}{479--486}).
\newblock \bibinfo{publisher}{Morgan Kaufmann}.
%Type = Inproceedings
\bibitem[{Rintanen(2004)}]{RintanenICAPS04}
\bibinfo{author}{Rintanen, J.} (\bibinfo{year}{2004}).
\newblock \bibinfo{title}{Complexity of planning with partial observability}.
\newblock In {\it \bibinfo{booktitle}{{ICAPS}}\/} (pp.
  \bibinfo{pages}{345--354}).
%Type = Inproceedings
\bibitem[{Sohrabi et~al.(2011)Sohrabi, Baier \&
  McIlraith}]{sohrabi2011preferred}
\bibinfo{author}{Sohrabi, S.}, \bibinfo{author}{Baier, J.}, \&
  \bibinfo{author}{McIlraith, S.} (\bibinfo{year}{2011}).
\newblock \bibinfo{title}{Preferred explanations: Theory and generation via
  planning}.
\newblock In {\it \bibinfo{booktitle}{{AAAI}}\/} (pp.
  \bibinfo{pages}{261--267}).
\newblock volume~\bibinfo{volume}{25}.
%Type = Article
\bibitem[{Thi{\'e}baux et~al.(2005)Thi{\'e}baux, Hoffmann \&
  Nebel}]{thiebaux2005defense}
\bibinfo{author}{Thi{\'e}baux, S.}, \bibinfo{author}{Hoffmann, J.}, \&
  \bibinfo{author}{Nebel, B.} (\bibinfo{year}{2005}).
\newblock \bibinfo{title}{In defense of pddl axioms}.
\newblock {\it \bibinfo{journal}{Artificial Intelligence}\/},  {\it
  \bibinfo{volume}{168}\/}, \bibinfo{pages}{38--69}.
%Type = Inproceedings
\bibitem[{Torres \& Baier(2015)}]{torres2015polynomial}
\bibinfo{author}{Torres, J.}, \& \bibinfo{author}{Baier, J.~A.}
  (\bibinfo{year}{2015}).
\newblock \bibinfo{title}{Polynomial-time reformulations of {LTL} temporally
  extended goals into final-state goals}.
\newblock In {\it \bibinfo{booktitle}{ICAPS}\/}.
%Type = Inproceedings
\bibitem[{Zhu et~al.(2019)Zhu, Pu \& Vardi}]{zhu2019first}
\bibinfo{author}{Zhu, S.}, \bibinfo{author}{Pu, G.}, \& \bibinfo{author}{Vardi,
  M.~Y.} (\bibinfo{year}{2019}).
\newblock \bibinfo{title}{First-order vs. second-order encodings for
  ltlf-to-automata translation}.
\newblock In {\it \bibinfo{booktitle}{{TAMC}}\/} (pp.
  \bibinfo{pages}{684--705}).

\end{thebibliography}

\end{document}